\documentclass[letterpaper]{article} 
\usepackage{aaai2026}  
\usepackage{times}  
\usepackage{helvet}  
\usepackage{courier}  
\usepackage[hyphens]{url}  
\usepackage{graphicx} 
\usepackage{natbib}  
\usepackage{caption} 
\usepackage{algorithm}
\usepackage{algorithmic}

\usepackage{newfloat}
\usepackage{listings}
\DeclareCaptionStyle{ruled}{labelfont=normalfont,labelsep=colon,strut=off} 
\floatstyle{ruled}
\newfloat{listing}{tb}{lst}{}
\floatname{listing}{Listing}
\title{On the Exponential Convergence for Offline RLHF with Pairwise Comparisons}
\author{
    Zhirui Chen\textsuperscript{\rm 1},
    Vincent Y. F. Tan\textsuperscript{\rm 1}
}
\affiliations{
    \textsuperscript{\rm 1}National University of Singapore\\
    zhiruichen@u.nus.edu, vtan@nus.edu.sg
}
\usepackage{enumitem}

\usepackage{subcaption}
\usepackage{microtype}
\usepackage{graphicx}
\usepackage{booktabs} 
\usepackage{colortbl}
\usepackage{pifont}
\usepackage{algorithm}
\usepackage{algorithmic}

\usepackage{amsmath,mathrsfs,mathtools,amsfonts,amssymb,amsthm}
\usepackage[capitalize,noabbrev]{cleveref}
\renewcommand{\algorithmicrequire}{ \textbf{Input:}}
\renewcommand{\algorithmicensure}{ \textbf{Output:}}
\allowdisplaybreaks
\theoremstyle{plain}
\newtheorem{theorem}{Theorem}[section]
\newtheorem{proposition}[theorem]{Proposition}
\newtheorem{lemma}[theorem]{Lemma}
\newtheorem{corollary}[theorem]{Corollary}
\theoremstyle{definition}
\newtheorem{definition}[theorem]{Definition}
\newtheorem{assumption}[theorem]{Assumption}
\theoremstyle{remark}

\newtheorem*{rep@theorem}{\rep@title}
\newcommand{\newreptheorem}[2]{%
\newenvironment{rep#1}[1]{%
 \def\rep@title{\textbf{#2 \ref{##1}}}%
 \begin{rep@theorem}}%
 {\end{rep@theorem}}}
\makeatother

\newreptheorem{theorem}{Theorem}
\newreptheorem{lemma}{Lemma}

\usepackage[textsize=tiny]{todonotes}
\usepackage{comment}
\mathchardef\mhyphen="2D
\DeclareMathOperator*{\argmax}{argmax}
\DeclareMathOperator*{\argmin}{argmin}

\newcommand\given[1][]{\:#1\vert\:}

\usepackage[utf8]{inputenc} 
\usepackage[T1]{fontenc}    
\usepackage{url}            
\usepackage{booktabs}       
\usepackage{amsfonts}       
\usepackage{nicefrac}       
\usepackage{xcolor}         
\usepackage{enumitem}
\usepackage{lscape}
\usepackage{tabularx}
\usepackage{dsfont}

\begin{document}

\maketitle

\begin{abstract}
    We consider the problem of offline reinforcement learning from human feedback (RLHF) with pairwise comparisons  proposed by~\citet{zhu2023principled},  where the implicit reward is a linear function of an unknown parameter. Given an offline dataset, our objective consists in ascertaining the optimal action for each state, with the ultimate goal of minimizing the {\em simple regret}. We propose an algorithm, \underline{RL} with \underline{L}ocally \underline{O}ptimal \underline{W}eights or {\sc RL-LOW}, which yields an exponential form of simple regret of $\exp ( - \Omega(n/H) )$ where $n$ is the number of data samples and $H$ denotes an instance-dependent hardness quantity that depends explicitly on the suboptimality gap  of each action.  Furthermore, we derive a first-of-its-kind instance-dependent lower bound in offline RLHF with pairwise comparisons. Interestingly, we observe that the lower and upper bounds on the simple regret match order-wise in the exponent, demonstrating order-wise optimality of our {\sc RL-LOW}. 
    In view of privacy considerations in practical applications, we also extend {\sc RL-LOW} to the setting of $(\varepsilon,\delta)$-differential privacy and show, somewhat surprisingly, that the hardness parameter $H$ is unchanged in the asymptotic regime as $n$ tends to infinity; this underscores the inherent efficiency of {\sc RL-LOW} in terms of preserving the privacy of the observed rewards. Given our focus on establishing instance-dependent bounds of exponential convergence, our research fills the research gap in existing studies that concentrate on establishing worst-case regrets of {\em inverse polynomial convergence} (e.g., $\widetilde{O}(\frac{1}{\sqrt{n}})$) for offline RLHF with pairwise comparisons.
\end{abstract}


\section{Introduction}
Reinforcement Learning (RL)~\citep{sutton2018reinforcement} has been widely recognized for its capacity to facilitate agents in learning  a sequence of optimal actions through iterative interactions with their environments. However, RL encounters significant hurdles in environments that are characterized by uncertainty or lacking explicit reward signals. To address these shortcomings, the concept of RL with human feedback (or RLHF) has emerged as a prominent paradigm. Preference-based RL (PbRL)~\citep{christiano2017deep,chen2022human,ibarz2018reward,palan2019learning}   has stood out as one of the most widely used frameworks for RLHF.  In this regard, preference-based RL has achieved remarkable performances in practical applications, with particular importance lying in its ability to align large language models (LLMs) with human intent, thereby mitigating the output of toxic and dishonest information~\citep{ouyang2022training, ziegler2019fine,glaese2022improving, bai2022training, liu2023chain}, and improving the quality of applying to  the specific tasks~\citep{stiennon2020learning,wu2021recursively,nakano2021webgpt}. 
 
 In this work, we tackle the problem of offline RLHF with pairwise comparisons, wherein the learning mechanism operates solely on pre-existing (or offline) data without dynamically engaging with the environment. Given the high cost associated with human interaction, offline RLHF has assumed particular importance in the context of incorporating human feedback. The significance of this offline framework has been justified by many previous  prominent works~\citep{shin2023benchmarks,ouyang2022training,zhu2023principled,kim2023preference}. For instance, within the learning process of InstructGPT~\citep{ouyang2022training} or {\color{black} the training procedure of~\citet{ahmadian2024back}}, a pivotal procedure involves the training of a reward model utilizing pre-trained LLM feature vectors, coupled with the utilization of pre-collected human pairwise comparisons as the training dataset. Conceptually, this procedure can be construed as treating the current prompt context as a state within a certain Markov Decision Process (MDP), while the responses generated by the LLM serve as actions within this process. Empirical findings presented by~\citet{ouyang2022training} demonstrate the efficacy of this offline framework in effectively aligning human intent with the outputs of LLMs.
 
 However, the literature concerning theoretical analyses within the domain of offline PbRL remains rather scant. Previous theoretical analyses~\citep{zhu2023principled,zhan2024provable} of offline PbRL predominantly focused on the worst-case (or minimax) regret, often resulting in the derivation of regret upper bounds for their algorithms of the form $\tilde{O}(n^{-1/2})$, where $n$ is the size of the offline dataset. In this work, we adopt a different  approach that is centered  on instance-dependent guarantees. In other words, we wish to derive performance guarantees that are functions of the specific problem instance, thus elucidating the role of fundamental hardness parameters. This yields  complementary insights to the existing worst-case analyses. To this end, we design and analyze {\sc RL-LOW}, a preference-based RL algorithm. Our analysis of the performance {\sc RL-LOW} unveils an instance-dependent simple regret bound of $\exp(-\Omega(n/H))$, where $H$ is a hardness parameter. This reveals that the simple regret   decays exponentially fast in the size of the dataset $n$ and the exponential rate of convergence has also been identified. Complementarily, by proving an instance-dependent lower bound,  we show that any algorithm will suffer from a simple regret of at least $\exp( - O(n/H))$. Thus, the dependence of the problem on $H$ is fundamental and cannot be improved upon, thereby demonstrating the efficacy of {\sc RL-LOW} and the tightness of our analyses.
 

\subsection{Related Works}
\textbf{Preference-Based RL:}
From the  empirical viewpoint,  \citet{christiano2017deep} initially demonstrated that RL systems can effectively address complex tasks like Atari games and simulated robot locomotion by learning from human preferences between trajectory segments. Later, numerous researchers started to employ human pairwise comparisons to enhance the performance of LLMs, e.g., aligning the LLMs' behavior with human intent~\citep{ ziegler2019fine,glaese2022improving, bai2022training, liu2023chain}, and enhancing the efficacy of application to specific tasks~\citep{stiennon2020learning,wu2021recursively,nakano2021webgpt}.

From the  theoretical perspective,  the existing literature remains sparse in offline RLHF with pairwise comparisons. \citet{zhu2023principled} elucidated the failure of the maximum likelihood estimation (MLE) procedure in some scenarios. Motivated by this, they theoretically prove the (near) minimax optimality of the {\sc  Pessimistic MLE} approach with a high probability guarantee. In addition, ~\citet{zhan2024provable} introduced a novel paradigm for general reward functions, and they introduce  $\varepsilon$-bracket approximations for reward models, accompanied by a rigorous theoretical analysis delineating sample complexity in terms of approximation error $\varepsilon$ and the high-probability parameter $\delta$. Recently, despite the significant contributions  of~\citet{cen2024value} and \citet{liu2024provably} in advancing the integration of experimental findings and theoretical analysis in offline RLHF, the bounds they established are still characterized as worst-case bounds with inverse polynomial forms.

We observe that the above theoretical investigations, while invaluable, are not instance-dependent and do not exhibit exponential convergence.  Typically, the above minimax or worst-case guarantees yield  upper bounds in the form of $\tilde{O}(n^{-1/2})$ and do not depend on  any problem-specific factors (such as  suboptimality gaps). Our research stands out as a pioneering attempt in offering an instance-dependent examination that yields exponential convergence for offline RLHF with pairwise comparisons, thereby bridging a critical gap in the existing literature. 

\textbf{Label-Differential Privacy:}  In our study, we also consider the notion of {\em label privacy}, acknowledging that the labels in our offline dataset originate from users, thus highlighting the imperative to protect user privacy. 
 \citet{chaudhuri2011sample} were among the pioneers in exploring the concept of {\em label privacy} within the context of supervised learning for binary classification. Their foundational work posits that the sensitive information primarily resides in the labels, while considering the unlabeled attributes as non-sensitive. Later, the concept of label privacy has been investigated across various machine learning paradigms, including but not limited to PAC learning~\citep{beimel2013private} and deep learning frameworks~\citep{ghazi2021deep}. This broadened examination underscores the significance and relevance of label privacy considerations across diverse areas of machine learning research and applications.

More recently,  \citet{chowdhury2024differentially} investigated the use of label differential privacy to protect the privacy of human labelers in the process of estimating rewards from preference-based feedback. \citet{chowdhury2024differentially} derive an upper bound for their proposed algorithm on the estimation error. They show that it also decays as $O(n^{-1/2})$ and the implied constant  here depends on $(\varepsilon,\delta)$, the parameters that define differential privacy. 
This bound only applies in the scenario of estimating the reward value and is not applicable if we want to understand how it depends on the simple regret of a specific instance. In our work, we consider the effect of $(\varepsilon,\delta)$-DP on the simple regret.

\subsection{Our Contributions}
We summarize our main contributions as follows: 
\begin{enumerate}[wide, labelwidth=!, labelindent=0pt]
    \item We establish the first-of-its-kind instance-dependent lower bound characterized by suboptimality gaps for a given problem instance. Our analysis reveals that this lower bound takes the   form   $\exp(-O( {n}/{H}))$, where $H$ is a hardness parameter that is an explicit function of the suboptimality gaps.  This finding furnishes a novel, and possibly generalizable,  analytical approach for assessing algorithmic performance within the realm of preference-based RL.

    \item  We design a simple algorithm {\sc RL-LOW} based on the novel concept of {\em locally optimal weights}. Our analysis demonstrates that its expected simple regret matches the aforementioned instance-dependent lower bound (in the exponential decay rate of the simple regret), thus revealing our algorithm's achievement of instance-dependent optimality. 
    \item We extend {\sc RL-LOW} to be applicable to the $(\varepsilon,\delta)$-differential privacy {\color{black} with labels} by combining the Gaussian mechanism with  the aforementioned locally optimal weights. Our analysis demonstrates that, for large datasets, this combination enables our algorithm to achieve differential privacy without weakening the bound on the simple regret,  underscoring the superiority of the design and analysis of  {\sc RL-LOW}.
    \item As a by-product of our analyses, we show that   {\sc RL-LOW} achieves a worst-case bound of the form $O(n^{-1/2})$ for the dependency of $n$. If we translate the high-probability upper bound in~\citet{zhu2023principled} to the same worst-case setting, we obtain a bound of the form $O(\sqrt{n^{-1}\log n } )$. Thus, our work provides a noticeable (albeit small) improvement over the state-of-the-art theoretical result in~\cite{zhu2023principled} in terms of the dependency on the sample size $n$. 
\end{enumerate}

\section{Preliminaries and Problem Setup}

Let $\mathcal{S}=\{1,\ldots,S\}$ denote the state space, and $\mathcal{A}=\{1,\ldots,A\}$ denote the action set. The $i$-th action of state $k$ is associated with the feature vector $\phi(k,i) \in \mathbb{R}^d$, and its associated (unknown) reward is 
\begin{equation}
r_{k,i} = \langle \phi(k,i), \theta \rangle,     \label{eqn:rew_linear}
\end{equation}
where $\theta \in \mathbb{R}^d$ is an unknown parameter vector. The collection of all feature vectors is denoted as $\phi = \{ \phi(k,i) \}_{k \in \mathcal{S}, i \in \mathcal{A}}$. For all $k\in \mathcal{S}$, we denote the suboptimaliy gap of action $i \in \mathcal{A}$ as $\Delta_{k,i}= \max_{j\in \mathcal{A}} r_{k,j} - r_{k,i}$. Let $(a^{(0)},a^{(1)}) \in \mathcal{A}^2$ be a pair of comparisons and let  $s\in \mathcal{S}$ be a state. Then, we define a stochastic label $\sigma \in \{0,1\}$, following the Bradley--Terry--Luce (BTL) model as
\begin{equation}
    \mathbb{P}\big(\sigma =1 \mid a^{(0)},a^{(1)}, s \big)= \frac{\exp\big( r_{s, a^{(1)} } \big) }{\exp\big( r_{s, a^{(0)} } \big)+ \exp\big( r_{s, a^{(1)} } \big)}. 
\end{equation}
Given this model, we assume throughout that we have access to an {\em offline dataset}, which we denote as 
  $\mathcal{D}=\{({s}_i,{a}^{(0)}_i,{a}^{(1)}_i,\sigma_i)\}_{i=1}^n$. Note that this dataset consists of $n$ tupies of states,   pairs of actions for comparison, and stochastic labels.  Without loss of generality, we assume that the comparisons are arranged such that $a^{(0)}_i < a^{(1)}_i$ for all $i=1,\ldots,n$, and $a^{(0)}_i < a^{(0)}_j$ (or $a^{(1)}_i \le a^{(1)}_j$ if $a^{(0)}_i = a^{(0)}_j$) for all $i<j$. 
For simplicity,  we assume that the feature vectors satisfy $\phi(k,i)\neq \phi(k,j)$ for all states $k\in\mathcal{S}$ and all actions $i\ne j$. In addition, we assume that   for each state $k$, the best action $i_k^*  =\arg\max_{j \in \mathcal{A}} r_{k,j}$ is unique.   
Broadly speaking, our objective is to use the offline dataset $\mathcal{D}$ to estimate the best action $i_k^*$ for each state $k \in \mathcal{S}$. Following ~\citet{zhu2023principled}, we aim to design a (possibly randomised) algorithm  $\Pi$ that uses the dataset $\mathcal{D}$ to output a set of actions $\{\hat{i}_k \}_{k\in \mathcal{S}}$ that minimizes the  {\em  simple regret}\footnote{The term  ``simple regret'' is referred to as ``performance gap'' in some existing works (e.g.,~\citet{zhu2023principled}). }, defined as 
\begin{equation}
R_n =  \mathbb{E}_{k\sim\rho} \big[  r_{k,i^*_k} - r_{k,\hat{i}_k} \big],
\end{equation}
where $\rho=(\rho_1,\ldots,\rho_S)$ is an unknown static distribution over states. Without loss of generality, we assume $\rho_i>0$ for $i\in \mathcal{S}$.
{\color{black} We also consider a generalised version of the regret that is amenable to the MDP setting of RL in Section~\ref{sec:MDP_EXT}.} by extending  $\rho$ to be dynamic. 
%
Let $N\in \mathbb{R}^{S \times A \times A}$ be a tensor that collects the proportions of each comparison in the dataset $\mathcal{D}$, which satisefies $\lceil n N_{k,i,j} \rceil =  \sum_{\iota=1}^n \mathbf{1} \{ s_\iota=k,a_\iota^{(1)}=i,a_\iota^{(2)}=j  \}$, where $N_{k,i,j} \in [0,1]$ is the proportion of the number of times actions $i$ and $j$ have been compared under state $k$.  In this paper, we consider the concept of problem instance, denoted as $v$, which is characterised by $v=(\rho,\mathcal{S},\mathcal{A},\phi ,N, \theta)$. 

In the following, we index instance-specific parameters with the instance $v$ to indicate their dependence on $v$; this will be omitted when the instance is clear from the context. In addition, we write $\mathbb{P}_v^\Pi$ (resp. $\mathbb{E}_v^\Pi$) to denote the probability measure (reps. the expectation) induced under algorithm $\Pi$ and under the instance $v$. For notational brevity in the rest of the paper, we assume $nN_{k,i,j}\in \mathbb{N}$ is an integer. \footnote{For notational brevity, we assume that $nN_{k,i,j}\in \mathbb{N}$ is an integer. To be more precise, the sample count for $(k,i,j)$   should be written as $\lceil nN_{k,i,j} \rceil$.} 

\begin{assumption}{(Bounded Reward)}
\label{assump:bounded_reward}
There exists a finite and known constant $L$ such  that for any $k\in \mathcal{S}$ and $i \in \mathcal{A}$, it holds that $\lvert \langle \phi(k,i), \theta \rangle \rvert \le L$. 
\end{assumption}
In previous works~\citep{zhu2023principled}, the authors assume that the norms of the feature vectors $\phi(k,i)$ and parameter vector $\theta$ are separately bounded. This  clearly implies that Assumption~\ref{assump:bounded_reward} is satisfied, but Assumption~\ref{assump:bounded_reward} is weaker as it is a bound on the rewards. 
\begin{definition}(Consistent Instance) \label{def:consist}
 A problem instance $v=(\rho,\mathcal{S},\mathcal{A},\phi,N, \theta)$ is {\em consistent} if for all $(k,i,j) \in \mathcal{S}\times \mathcal{A}^2$, it holds that
$
\phi(k,i) - \phi(k,j) \in {\rm Span} \{\phi(k',i')-\phi(k',j') : (k',i',j')\in \mathcal{S}\times \mathcal{A}^2 \text{ and } N_{k',i',j'}>0  \}.
$ 
\end{definition}
We say an instance $v$ is {\em inconsistent} if it is not consistent.  In the following, we will be only concerned with those instances that are consistent as the following result  shows that 
 it is impossible to design a algorithm that achieves vanishing simple regret for inconsistent instances. 

\begin{proposition}(Impossibility Result)
\label{prop:imposs_results}
For any inconsistent instance $v=(\rho,\mathcal{S},\mathcal{A},\phi,N, \theta)$, there exists an instance $v'=(\rho,\mathcal{S},\mathcal{A},\phi,N, \theta')$ such that for all algorithms $\Pi$
\begin{equation}
 \liminf_{n \rightarrow \infty} \big\{\mathbb{E}^{\Pi}_v [R_n] +\mathbb{E}^{\Pi}_{v'} [R_n] \big\}>0.   
\end{equation}
\end{proposition}

\section{The Proposed Algorithm: {\sc RL-LOW}}

In this section, we describe our computationally and statistically efficient algorithm for offline RLHF with pairwise comparisons based on the novel idea of {\em locally optimal weights} for estimating the relative reward of each pair of actions.  For clarity in  exposition, this section is devoted to the setting of a  static state distribution $\rho$; this is done as a foundational step before we extend the ideas to the MDP setting of RL  in Section~\ref{sec:MDP_EXT}.

Our proposed algorithm, called {\sc RL-LOW}, is simple and is presented formally in Algorithm~\ref{alg:client1}. Before we describe its components, we introduce some notations.

Let $B_{k,i,j}$ be the empirical success rate with the comparison of $i$ and  $j$, i.e., for $k\in \mathcal{S}$ and $i,j\in \mathcal{A}$ with $N_{k,i,j}>0,$
{\small 
\begin{equation}
\label{eq:def_B}
B_{k,i,j} \coloneqq \frac{1}{nN_{k,i,j}} \sum_{\iota=1}^n \sigma_\iota \mathds{1} \{s_\iota=k,a_\iota^{(1)}=i,a_\iota^{(2)}=j\},
\end{equation} 
}
and $B_{k,j,i} \coloneqq 1-B_{k,i,j}$. If $N_{k,i,j}=N_{k,j,i}=0$, we define $B_{k,i,j}=B_{k,j,i}=0$.
Subsequently, certain empirical success rates may exhibit magnitudes that are either excessively large or   small. We clip them by means of the following operation: ${B}_{k,i,j}^{\rm CLP}  = \mathrm{CLIP}_L({B}_{k,i,j})$, where
\begin{equation}
\label{def:clip_B}
 \ \mathrm{CLIP}_L (a) 
=\left\{\begin{array}{cc}
     \frac{\exp(2L)}{1+\exp(2L)}&  a > \frac{\exp(2L)}{1+\exp(2L)} \\
     \frac{1}{1+\exp(2L)}  &  a<   \frac{1}{1+\exp(2L)} \\
     a & \text{otherwise}
\end{array}\right. .
\end{equation}
Per   Assumption~\ref{assump:bounded_reward}, the implicit rewards are bounded  by $L$. Consequently, within our BTL model framework, the success rate of each comparison necessarily falls within the interval $\big[\frac{1}{1+\exp(2L)}, \frac{\exp(2L)}{1+\exp(2L)}\big]$. We exploit this in Eqn.~\eqref{def:clip_B} to ensure that the implementation of our clip operation is consistent with the model's dynamics. We are now ready to introduce the notion of {\em locally optimal} weights, which plays a central role in the estimation of the rewards.

\begin{algorithm}[t]
\caption{Reinforcement Learning with Locally Optimal Weights (RL-LOW)} \label{alg:client1}
\begin{algorithmic}[1]
\REQUIRE Dataset $\mathcal{D}=\{(s_i,a^{(1)}_i,a^{(2)}_i,\sigma_i)\}_{i=1}^n$ and feature maps $\phi = \{ \phi(k,i)  \}_{k \in \mathcal{S}, i \in \mathcal{A}}$ .\\
 
\ENSURE The estimated best action $\hat{i}_k \in\mathcal{A}$ for each  $k \in \mathcal{S}$.
 
 \STATE Compute the sample proportions  $N_{k,i,j} \leftarrow \frac{1}{n} \sum_{\iota=1}^n \mathds{1} \{s_\iota=k,a_\iota^{(1)}=i,a_\iota^{(2)}=j\}$.
 \STATE For $k\in \mathcal{S}$ and $i,j\in \mathcal{A}$ such that $N_{k,i,j}>0$,  compute the success rate $B_{k,i,j}$  using Eqn.~\eqref{eq:def_B}.

\STATE Compute ${B}_{k,i,j}^{\rm CLP}$ by clipping  $B_{k,i,j}$ through Eqn.~\eqref{def:clip_B}. 
 
 \STATE For each state $k\in \mathcal{S}$ and distinct actions $i,j \in \mathcal{A}$ with $i<j$, compute the locally optimal weights $(w^{(k,i,j)}_{k',i',j'})_{k'\in \mathcal{S}, i',j'\in \mathcal{A}}$ using Eqn.~\eqref{eqn:opt_weights}.
   \STATE Compute the empirical relative reward $\hat{r}_{k,i,j}$  for  each $k\in \mathcal{S}, i , j\in \mathcal{A}$  using Eqn.~\eqref{eqn:estimate_reward}.
\STATE RETURN for any $k\in \mathcal{S}$, let $\hat{i}_k \in \{i \in \mathcal{A}: \hat{r}_{k,i,j} \ge 0, \forall j\neq i \}$; resolve ties uniformly.
\end{algorithmic}
\end{algorithm}

\begin{definition}{(Locally Optimal Weights)}
\label{def:local_optimal_weights}
For an consistent instance $v$, let $ \mathcal{U}_{k,i,j}= \{u \in \mathbb{R}^{S\times A\times A}:  \phi(k,i)-\phi(k,j)= \sum_{k'\in \mathcal{S}, i',j'\in \mathcal{A}} u_{k',i',j'}(\phi(k',i')-\phi(k',j'))  \;\mbox{and}\; u_{k',i',j'}=0 \;\mbox{if}\;   N_{k',i',j'}=0 \}$. We say that $w^{(k,i,j)}= (w_{k',i',j'}^{(k,i,j)})_{k' \in \mathcal{S}, i', j' \in \mathcal{A} } $ is a  set of {\em locally optimal weights} for $(k,i,j)\in \mathcal{S}\times \mathcal{A}^2$ with $i\neq j$ if 
{\small
\begin{equation}
w^{(k,i,j)} \in \argmin_{u \in\mathcal{U}_{k,i,j} } \bigg\{\sum_{k' \in \mathcal{S},i',j'\in \mathcal{A} :N_{k',i',j'}> 0} \frac{  (u_{k',i',j'})^2 }{N_{k',i',j'}}\bigg\} .  \label{eqn:opt_weights}
\end{equation} 
}
\end{definition}
The weights in Eqn.~\eqref{eqn:opt_weights} are   described as  ``locally optimal'' because they are specifically customized to each $(k,i,j)$ tuple. Hence, $w^{(k,i,j)}$ is {\em  local} to $(k,i,j)$. This is a novelty in the design of our algorithm.

By the definition of the consistency of an instance (cf.\ Definition~\ref{def:consist}), there exists a subset $\beta\subset\mathcal{S}\times\mathcal{A}^2$ such that $
\phi(k,i) - \phi(k,j) \in {\rm Span} \{\phi(k',i')-\phi(k',j') : (k',i',j') \in \beta \}.
$ Hence, there   exists a locally optimal weight for every pair of actions given a consistent instance. In addition, $w^{(k,i,j)}$ can be calculated efficiently by its analytic form (see details in Appendix~\ref{sec:existing_i_hat}\footnote{Following AAAI policy, the appendices are provided on the arxiv~\citep{chen2024order}.}).

Equipped with the definition of locally optimal weights, we now provide an estimate of the relative reward for state $k \in \mathcal{S}$ and pair of action $(i,j)\in \mathcal{A}^2$ with $i \neq j$ as follows: 
\begin{equation}
\hat{r}_{k, i, j}=  \sum_{k'\in \mathcal{S}, i'\in \mathcal{A},j'\in \mathcal{A} } w^{(k,i,j)}_{k',i',j'}\, \log\Big(\frac{B^{\rm CLP}_{k',i',j'}}{1-B^{\rm CLP}_{k',i',j'}} \Big),\label{eqn:estimate_reward}
\end{equation}
and we define $\hat{r}_{k,i,i}=0$ for all $k\in \mathcal{S}$ and $i \in \mathcal{A}$.
The term $\frac{(u_{k',i',j'})^2  }{N_{k',i',j'} }$ in Eqn.~\eqref{eqn:opt_weights} is  a proxy for the variance introduced by the pair of actions  $(i',j')$ in state $k'$ when associated with the coefficient $u_{k',i',j'}$  in the linear combination of the definition of $\mathcal{U}_{k,i,j}$. Our objective is to minimize the cumulative   variance proxy for $(k,i,j)$, thus enhancing the precision of the  estimate of the relative reward for $(k,i,j)$ for the purposes of establishing the tightest possible concentration result for   subGaussian random variables  (see  Appendix~\ref{sec:proof_upperbound}).

 Finally, for any $k\in \mathcal{S}$, let $\hat{i}_k \in \hat{\mathcal{I}}_k:=\{i \in \mathcal{A}: \hat{r}_{k,i,j} \ge 0, \forall\, j\neq i \}$ be any estimate of the best action under state $k$. It is natural to wonder whether $\hat{i}_k$ exists, i.e., whether the set  $\hat{\mathcal{I}}_k$ is empty. The following proposition answers this in the affirmative. 
\begin{proposition}
\label{propo:transitivity}
     For any consistent instance $v$ and using estimate of the best action $\hat{i}_k$ under each state $k$ as prescribed by {\sc RL-LOW}, we have $\lvert \hat{\mathcal{I}}_k\rvert \ge 1$ and
     \begin{equation*}
         \argmax_{i\in \mathcal{A}} \hat{r}_{k,i,j_1} = \argmax_{i\in \mathcal{A}} \hat{r}_{k,i,j_2} = \hat{\mathcal{I}}_k \quad \mbox{for any}\ j_1,j_2\in \mathcal{A}.
     \end{equation*}
\end{proposition}
\paragraph{Computational Complexity:} Proposition~\ref{propo:transitivity} obviates the need to compute all values of $\hat{r}_{k,i,j}$ for each $(k,i,j) \in \mathcal{S} \times \mathcal{A}^2$. We demonstrate that the {\sc RL-LOW} algorithm can be efficiently implemented with a computational complexity of $\mathcal{O}(SAd + nd^2 + d^3)$, as the term $SAd $ corresponds to the natural process of scanning the feature vectors for all state-action pairs. The terms $ nd^2 + d^3 $ are typical in scenarios involving a linear reward structure, such as in linear regression. {\color{black} It is worth noting that the term $SAd$  can be removed if we do not need to output $\hat{i}_k$ for   each $k\in \mathcal{S} $, but rather a {\em parametric function} $\hat{i}(k;\vartheta)$ is to be learned, and the overall  computational complexity  becomes $\mathcal{O}(nd^2 + d^3)$; see details in Appendix~\ref{sec:existing_i_hat}.} 

\subsection{Upper Bound of {\sc RL-LOW}}
In this section, we provide an instance-dependent upper bound of the {simple regret} for {Rl-LOW}, and we also provide a worst-case upper bound as a by-product. 
First, we define an instance-dependent hardness parameter $H(v)$. 
Let 
{ 
\small
\begin{equation}
       H(v) \coloneqq \max_{k\in \mathcal{S}, i \in \mathcal{A}: i \neq i^*_k}  \   \frac{\gamma_{k,i}}{\Delta_{k,i}^2},
 \end{equation}
\begin{equation*}
 \mbox{where }   \gamma_{k,i} \coloneqq \sum_{k' \in \mathcal{S},i',j'\in \mathcal{A} :N_{k',i',j'}> 0} \frac{  (w^{(k,i,i_k^*)}_{k',i',j'} )^2 }{N_{k',i',j'}} . \label{eq:hardness}
\end{equation*}
}
The parameter $\gamma_{k,i}$ exhibits a positive correlation with the variance proxy of the relative empirical reward $\hat{r}_{k,i,i_k^*}$. Consequently, the ratio $\frac{\gamma_{k,i}}{\Delta_{k,i}^2}$   in the definition of $H(v)$ serves as a quantitative measure of the difficulty that the empirical reward of a  suboptimal action $i$ surpasses that of    $i^*_k$ in state $k$; {\color{black} see more intuitive explanations in Appendix~\ref{sub:hardness_intuition}. }

 
\begin{theorem}{(Instance-Dependent Upper Bound)} 
\label{thm:upper_bound}
For any consistent instance $v$, under {\sc {RL-LOW}}, we have for all sufficiently large $n$,
{\small
\begin{equation}
 \mathbb{E}^{\rm RL\mbox{-}LOW}_v \left[R_n \right] \le \exp\bigg(-\frac{n }{C_{\rm up} \cdot H(v)}  \bigg),   
\end{equation}
}
where $C_{\rm up}$ is a universal constant.\footnote{In this paper, our universal constants depend on $L$, which is known and fixed throughout.}
\end{theorem}
 
From Theorem~\ref{thm:upper_bound}, it is evident that the upper bound decays exponentially fast and the exponent is a function of an instance-dependent hardness term $H(v)$.  This is the first instance-dependent analysis in offline reinforcement learning with pairwise comparisons. 
 
It is natural to wonder why we do not devise an instance-dependent analysis of or modification to the existing {\sc Pessimistic MLE} in \citet{zhu2023principled}. Note that {\sc Pessimistic MLE} is designed to perform well {\em with high probability} and not necessarily {\em in expectation}. In particular,  the regret  bound of {\sc Pessimistic MLE} holds with probability at least $1-\delta$. Hence, to ensure the regret  is less than $\exp(-\Omega(n/H(v)))$, one should set the failure probability $\delta$ to be $\exp(-\Theta(n/H(v)))$, which is not possible as $H(v)$ is unknown to the algorithm (since $\theta$ is also unknown). 
 We further provide a worst-case upper bound for {\sc RL-LOW} as follows.
\begin{proposition}{(Worst-Case Upper Bound)}
\label{props:non_problem_specific_bound}
For any consistent instance $v$ and for all $n\ge 1$, 
{\small
 \begin{equation}
 \!\mathbb{E}^{\rm RL\mbox{-}LOW}_v \left[R_n \right] \!\le\! \frac{\sum\limits_{ \substack{ k ,i  : i \neq i^*_k} }  \rho_k (\sqrt{\gamma_{k,i}}\!+\! \tilde{\gamma}_{k,i}) }{C_{\rm wup}\sqrt{n}} \;\,
 \end{equation}
 \begin{equation*}
 \mbox{where}\;\,  \tilde{\gamma}_{k,i}  \!=\!  \sum_{\substack{ k' , i' ,j':   N_{k',i',j'}> 0 }  } \frac{ \lvert w^{(k,i,i_k^*)}_{k',i',j'} \rvert }{\sqrt{N_{k',i',j'}}} 
 \end{equation*}
 }
 and $C_{\rm wup}>0$ is a universal constant.
\end{proposition}

We note that in~\citet{zhu2023principled}, the high probability upper bound is of the form  $O\big(\sqrt{n^{-1}\log(1/\delta)}\big)$ for the dependency of $n$ and $\delta$. Hence, if we desire a bound in expectation, we obtain, through the law of total probability, a bound of the form $\mathbb{E}[R_n]=  O\big(\sqrt{n^{-1}\log(1/\delta)}+ \delta\big)$. Minimizing this bound over $\delta$ yields $\mathbb{E}[R_n]=O\big(\sqrt{n^{-1}\log n } \big)$. {\color{black} 
In terms of the dependence on $n$, it exhibits a performance that is slightly inferior to our established upper bound $O(n^{-1/2})$ of {\sc RL-LOW}.}

\section{Instance-Dependent Lower Bound}
In this section, we derive the first-of-its-kind instance-dependent lower bound on offline RLHF with pairwise comparisons.
Before we present our bound, we present some auxiliary lemmas that are potentially instrumental in deriving lower bounds on other preference-based RL problems.

Given any instance $v$, we let $P_v^{(n)}$ denote the joint distribution of the associated labels $\{\sigma_i\}_{i=1}^n$. The following lemma provides an estimate of the Kullback--Leibler (KL) divergence between instances $v$ and $v'$ that share the same parameters except for the latent vector   $\theta$ as in Eqn.~\eqref{eqn:rew_linear}.
\begin{lemma}
\label{lemma:kl_divergence_between_instances}
For any instance $v=  (\rho,\mathcal{S},\mathcal{A},\phi,N, \theta)$ and $v'=  (\rho,\mathcal{S},\mathcal{A},\phi,N, \theta')$, it holds that
{ \small
\begin{align}
 2n\exp(-4R_{\max})
 \le  \frac{ D_{\rm KL}(P_v^{(n)} \| P_{v'}^{(n)})    }{\tilde{D}(v,v')} \le 2n\exp(2R_{\max} )  \nonumber
\end{align}
\begin{align}
 \text{where } \tilde{D}(v,v')=   \sum_{k \in \mathcal{S},i,j \in \mathcal{A}} N_{k,i,j} ( \langle \phi(k,i)-\phi(k,j), \theta-\theta' \rangle )^2 ,\label{eqn:defDtilde}
\end{align} 
}
\noindent and
 $R_{\max}=\max_{k \in \mathcal{S}, i\in \mathcal{A}} \max \{\lvert \langle \phi(k,i),\theta \rangle \rvert, \lvert \langle \phi(k,i),\theta' \rangle \rvert \}$ is the maximum absolute reward   in these two instances.
\end{lemma}
This lemma demonstrates that when the rewards are bounded, the weighted sum of squared differences of the relative rewards can be used to approximate the   KL  divergence between the distributions of two instances. The approximation is precisely $\tilde{D}(v,v')$  defined in Eqn.~\eqref{eqn:defDtilde}.
Furthermore, for any $\bold{z} \in \mathbb{R}^d$, $\eta \in \mathbb{R}$ and consistent instance $v=(\rho,\mathcal{S},\mathcal{A},\phi,N, \theta)$,  we let ${\rm Alt}(v,\bold{z},\eta)$ be the set of instances that share the same instance parameters except for $\theta$ and satisfies
$
  \langle \bold{z}, \theta'  -\theta \rangle  = \eta 
$
for all $v' = (\rho,\mathcal{S},\mathcal{A},\phi,N, \theta') \in  {\rm Alt}(v,\bold{z},\eta) $.
The following lemma states  a useful property that relates the ${\rm Alt}$  set to the ``approximate KL divergence''~$\tilde{D}$.

\begin{lemma}
\label{lemma:min_distance_main}
Let $\mathcal{G}$ be an arbitrary orthonormal basis of ${\rm Span} \{\phi(k',i')-\phi(k',j') : (k',i',j')\in \mathcal{S}\times \mathcal{A}^2 \text{ and } N_{k',i',j'}>0  \}$. Also let $[\mathbf{w}]_{\mathcal{G}}$ denote the column vector that represents $\mathbf{w}$ under basis~$\mathcal{G}$ \citep[Chapter 4]{meyer2000matrix}. Define the matrix 
{\small
\begin{equation}
\label{eq:define_V}
    V \coloneqq \sum_{k \in \mathcal{S},i,j \in \mathcal{A}} N_{k,i,j}  [\phi(k,i)-\phi(k,j)]_{\mathcal{G}} [\phi(k,i)-\phi(k,j)]_{\mathcal{G}} ^ \top.
\end{equation}
}
Then for any consistent instance $v=(\rho,\mathcal{S},\mathcal{A},\phi,N, \theta)$,  $\eta \in \mathbb{R}$, and $\bold{z} \in {\rm Span} \{\phi(k',i')-\phi(k',j') : (k',i',j')\in \mathcal{S}\times \mathcal{A}^2 \text{ and } N_{k',i',j'}>0  \}$, 
\begin{equation}
\label{eq:KL_with_alt}
\min_{v' \in {\rm Alt}(v,\bold{z},\eta)} \tilde{D}(v,v') = \frac{\eta^2}{\lVert [\bold{z}]_{\mathcal{G}} \rVert^2_{V^{-1}}}.
\end{equation}

\end{lemma}
Lemma~\ref{lemma:min_distance_main} provides an estimate of the KL divergence between instance $v$ and $v' \in {\rm Alt}(v,\bold{z},\eta) $. This, in turn, provides a convenient means to apply the ubiquitous  {\em change of measure} technique to  derive the lower bound.

In addition, let $(\bar{i}$, $\bar{k})$ be the  state-action pair that attains maximum in the definition of hardness in Eqn.~\eqref{eq:hardness}. Define the subset of instances
{\small
\begin{equation}
\mathcal{Q}= \bigg\{v \mbox{ consistent} :  \frac{\gamma_{\bar{k},\bar{i}}}{\Delta_{\bar{k},\bar{i}}^2} \ge  \frac{4\gamma_{k,i}}{\Delta_{k,i}^2} \;
\;\forall\,  (k,i)\neq (   \bar{k},\bar{i})\; \mbox{,}\; i \neq i^*_k \bigg\} . \label{eqn:defQ}
\end{equation}
}
 We are now ready to state our lower bound.
\begin{theorem}{(Instance-Dependent Lower Bound)}
\label{thm:lower_bound}
For any instance $v=(\rho,\mathcal{S},\mathcal{A},\phi ,N, \theta) \in \mathcal{{\mathcal{Q}}}$, there exists another instance $v'=(\rho,\mathcal{S},\mathcal{A},\phi,N, \theta')$ with  $  H(v)  \le H(v')\le 8H(v) $ such that for all sufficiently large $n$, 
$$
\inf_{\Pi} \big\{\mathbb{E}_v^\Pi [R_n] + \mathbb{E}_{v'}^\Pi [R_n ] \big\}\ge \exp\bigg(-\frac{ n}{C_{\mathrm{lo}}\cdot H(v)}\bigg),
$$
where $C_{\mathrm{lo}}>0$ is a universal constant.  
\end{theorem}
 
The alternative instance $v'$ that appears in   Theorem~\ref{thm:lower_bound} is judiciously chosen to be $v' \in {\rm Alt}(v,\phi(\bar{k}(v),\bar{i}(v))-\phi(\bar{k}(v),i^*_{\bar{k}}(v)), 2 \Delta_{\bar{k}(v),\bar{i}(v)}(v) )$. In particular, it is designed so that the optimal action $i^*_{\bar{k}}$ under state $k$ of instance $v$ will become suboptimal under instance $v'$, and its suboptimality gap is at least $\Delta_{\bar{k}(v),\bar{i}(v)}(v)$ under  $v'$.

Theorem~\ref{thm:lower_bound} is an instance-dependent lower bound for all instances in the set  $\mathcal{Q}$. The condition that defines $\mathcal{Q}$ in Eqn.~\eqref{eqn:defQ} ensures that the hardness quantities  $H(v)$ and $H(v')$ have the same order. 
Since instances in $\mathcal{Q}$ cover all possible hardness values $H(v)$ (i.e., for every hardness values $h>0$, there exists an instance in $\mathcal{Q}$ of hardness $h$), we conclude that for any (small) $\epsilon\in (0,1)$, 
there does not exist any algorithm $\Pi$ that achieves for all consistent instance $v$,
{\small
\begin{equation}
\mathbb{E}_v^\Pi [R_n]  = \exp\bigg(-\Omega\Big(\frac{n}{H (v)^{1-\epsilon}}\Big)\bigg) .     
\end{equation}
}
 In this sense, the exponential decay rate of the simple regret of {\sc RL-LOW} presented in Theorem~\ref{thm:upper_bound} is asymptotically tight (or optimal) and the exponential dependence on the hardness parameter $H(v)$ is necessary, fundamental, and cannot be improved upon. 

\section{Extension to the MDP Setting of RL}
\label{sec:MDP_EXT}
Similar to~\citet[Section 1]{zhu2023principled}, our definition of simple regret is based on the {\em static} state distribution $\rho$ in the previous sections. In this section, we extend our results to the MDP setting of RL when the transition probabilities $P(k'|k,i)$ for $(k',k,i)\in \mathcal{S}^2 \times \mathcal{A}$ are assumed to be known. This assumption is consistent with ~\citet[Section~5]{zhu2023principled}, and we also provide a motivational example for this assumption in Appendix~\ref{app:motivation}. Given the transition probabilities $P(k'|k,i)$  and an MDP policy $\pi$, we let $d^\pi$ denote the state distribution~{\color{black}\citep[Section 9.2]{sutton2018reinforcement}} under $\pi$. Without loss of generality, we assume the MDP policies are deterministic, and we denote $\pi(k) \in \mathcal{A}$ to be the output action of $\pi$ under state $k$.  Let $\pi^*$ denote the optimal  MDP policy that is assumed to be unique, i.e.,
$$
\pi^* = \argmax_\pi  \mathbb{E}_{k\sim d^\pi}  \big[ r_{k, \pi(k)} \big].
$$
Then, we define the {\em simple regret} of any MDP policy $\pi$ as
\begin{equation}
R^{\rm MDP}(\pi)=   \mathbb{E}_{k\sim d^{\pi^*}} \big[  r_{k,\pi^*(k)}] -\mathbb{E}_{k\sim d^\pi}  \big[ r_{k, \pi(k)} \big].
\end{equation}
We now adapt our {\sc RL-LOW} to the MDP setting by redefining the output as an MDP policy:
$$
\hat{\pi}_{\rm out} \in \argmax_{\pi} \mathbb{E}_{k \sim d^\pi} [\hat{r}_{k,\pi(k),j^\dag}],
$$
where $j^\dag \in \mathcal{A}$ is arbitrarily fixed (e.g., $j^\dag=1$).
We simply call this adaptation   {\sc RL-LOW-MDP}. The upper bound on its  simple regret is stated as follows.
\begin{theorem}{(Instance-Dependent Upper Bound for {\sc RL-LOW-MDP})} 
\label{thm:mdp}
Given any consistent instance $v$, for all sufficiently large~$n$,
{\small
\begin{equation*}
 \mathbb{E}^{\rm RL\mbox{-}LOW\mbox{-}MDP}_v \left[R^{\rm MDP}(\hat{\pi}_{\rm out} )\right] \le \exp \bigg(-  \frac{n}{C_{\rm MDP} \cdot H_{\rm MDP}(v)}  \bigg) 
\end{equation*}
}
 where $C_{\rm MDP}>0$ is a universal constant,   
 {\small
 \begin{equation*}
      H_{\rm MDP}(v) \coloneqq \max_{\pi \neq \pi^*} \frac{\gamma^{\rm MDP}(\pi) }{(\mathbb{E}_{k\sim d^{\pi^*}} \big[  r_{k,\pi^*(k)}] -\mathbb{E}_{k\sim d^\pi}  \big[ r_{k, \pi(k)} \big])^2}, 
\end{equation*}
 $$ \gamma^{\rm MDP}(\pi) \coloneqq \max_{k : \pi(k) \neq \pi^*(k)}\ \sum_{k',i',j' :N_{k',i',j'}> 0} \frac{ (w^{(k,\pi(k),\pi^*(k)}_{k',i',j'} )^2 }{N_{k',i',j'}}. $$
 }
 \end{theorem}
 In the presence of the MDP, $H_{\rm MDP}(v)$, which is a generalization of $H(v)$ in Eqn.~\eqref{eq:hardness}, turns out to be the instance-dependence hardness parameter of the problem. The proof of Theorem \ref{thm:mdp} is provided in Appendix \ref{sec:proof_upperbound}. 
 It is important to observe that there exist MDPs (e.g., $P(k|k,i)=1$ or $S=1$) such that
 Theorem~\ref{thm:mdp} particularizes to  Theorem~\ref{thm:upper_bound}. Moreover, the lower bound in Theorem~\ref{thm:lower_bound} is also applicable to the present more general MDP setting when the transition probability kernel $P( k'|k,i)$ is independent of $(k,i)$ and  $k'$ follows the distribution $\rho$. Admittedly, the complexity of the problem increases substantially when the transition probabilities are unknown; this aspect warrants further investigation in future studies. Our findings serve as an initial step in exploring instance-dependent bounds in offline RLHF.

\section{Extension to $(\varepsilon,\delta)$-Differential Privacy (DP)}

In this section, we extend our algorithm {\sc RL-LOW} to be amenable to   $(\varepsilon,\delta)$-differential privacy {\color{black} with labels}, and we provide a motivational example of this extension in Appendix~\ref{app:motivation}.
To formalize our results, we provide  the definition of $(\varepsilon,\delta)$-DP, following~\citet{dwork2014algorithmic}.
We say that two sets of preference labels, $\sigma:=\{\sigma_i\}_{i=1}^n$ and $\sigma':=\{\sigma'_i\}_{i=1}^n$  are {\em neighboring} if there exists $s \in [n]$ such that $\sigma_s \neq \sigma'_s$ and $\sigma_j = \sigma'_j$ for all $j \neq s$.
\begin{definition}{(Differential Privacy {\color{black} with labels})}
\label{def:DP}
Fix any label-free dataset  $\{(s_i,a^{(1)}_i,a^{(2)}_i)\}_{i=1}^n$. A (randomized) algorithm $\mathcal{M}:\{0,1\}^n \rightarrow  \mathcal{A}^{S}$ (that takes as inputs a set of labels and outputs a set of actions, one for each state) satisfies $(\varepsilon,\delta)$-DP if for all neighboring labels $\sigma:=\{\sigma_i\}_{i=1}^n$ and $\sigma':=\{\sigma'_i\}_{i=1}^n$ it holds $\forall\, \mathcal{Z} \subset \mathcal{A}^{S}$
\begin{equation}
    \mathbb{P} (\mathcal{M}(\sigma) \in \mathcal{Z}) \le e^\varepsilon \, \mathbb{P} (\mathcal{M}(\sigma^\prime)\in \mathcal{Z})+\delta.
    \label{eq:varepsilon-differential-privacy-definition}
\end{equation}
\end{definition}
Note that Definition~\ref{def:DP} primarily concerns protecting the privacy of users' {\em labels}. In particular, the DP mechanism guarantees that any alteration in a user's label must not substantially affect the output of our algorithm. Otherwise, there exists a risk that a user's label might be inferred through the algorithm's output. Our definition of differential privacy (DP) aligns with that of~\citet{chowdhury2024differentially}.

We now adapt our {\sc RL-LOW} to $(\varepsilon,\delta)$-DP by using the Gaussian mechanism ~\citep{dwork2014algorithmic}. Firstly, we introduce the private version of the empirical success rate (analogous to $B_{k,i,j}$   in Eqn.~\eqref{eq:def_B}), i.e., for all $k\in \mathcal{S}$ and $i,j\in \mathcal{A}$ and $N_{k,i,j}>0$,
{\small 
\begin{equation*}
\tilde{B}_{k,i,j} \coloneqq \frac{1}{nN_{k,i,j}} \sum_{\iota=1}^n \sigma_\iota \mathds{1}  \{s_\iota=k,a_\iota^{(1)}=i,a_\iota^{(2)}=j\} +\tilde{\xi}_{k,i,j},
\end{equation*} 
}

where $\tilde{\xi}_{k,i,j}$ is an independent (across $k$, $i$, and $j$) Gaussian noise with zero mean and variance $\frac{2\log(1.25/\delta)}{(\varepsilon nN_{k,i,j})^2}$, and we let $\tilde{B}_{k,j,i} \coloneqq 1-\tilde{B}_{k,i,j}$. If $N_{k,i,j}=N_{k,j,i}=0$, we define $\tilde{B}_{k,i,j}=\tilde{B}_{k,j,i}=0$. Again, analogously to the operation in Eqn.~\eqref{def:clip_B}, we clip $\tilde{B}_{k,i,j} $ to form
\begin{equation}
   \tilde{B}_{k,i,j}^{\rm CLP} =  \mathrm{CLIP}_L(\tilde{B}_{k,i,j}) 
\end{equation}
Similarly to Eqn.~\eqref{eqn:estimate_reward}, the   perturbed estimated relative rewards are given as follows
{\small
\begin{equation}
\tilde{r}_{k, i, j}=  \sum_{k'\in \mathcal{S}, (i' ,j')\in \mathcal{A}^2 } w^{(k,i,j)}_{k',i',j'}\, \log\Big(\frac{\tilde{B}^{\rm CLP}_{k',i',j'}}{1-\tilde{B}^{\rm CLP}_{k',i',j'}} \Big),\label{eqn:perturbed_estimate_reward}
\end{equation}
}
where $w^{(k,i,j)}$ is defined in Definition~\ref{def:local_optimal_weights}. Finally, the empirical best action is $\hat{i}_k \in\tilde{\mathcal{I}}_k:= \{i \in \mathcal{A}: \tilde{r}_{k,i,j} \ge 0, \forall\, j\neq i \}$. A similar argument as Proposition~\ref{propo:transitivity} shows that   $\hat{i}_k$ exists; see details in Appendix~\ref{sec:existing_i_hat} for the details. The algorithm described above is a differentially private version of {\sc RL-LOW} and hence, it is named {\sc DP-RL-LOW}.  

{\sc DP-RL-LOW} with the carefully chosen variance of $\xi_{k,i,j}$ fulfils the requirement of $(\varepsilon,\delta)$-DP.
\begin{proposition}
\label{prop:dp_guarrantee} Given privacy parameters $\varepsilon,\delta \in (0,1)$,   {\sc DP-RL-LOW}    satisfies  $(\varepsilon,\delta)$-DP.
\end{proposition}
The proof of Proposition~\ref{prop:dp_guarrantee} follows exactly along the lines of the proof of \citet[Theorem~A.1]{dwork2014algorithmic} and is omitted.  
Then, the upper bound is as follows.
\begin{theorem}{(Instance-Dependent Upper Bound for {\sc DP-RL-LOW})} 
\label{thm:dp_upperbound}
Given any consistent instance $v$, for all sufficiently large~$n$,
{\small
\begin{equation}
 \mathbb{E}^{\rm DP\mbox{-}RL\mbox{-}LOW}_v \left[R_n \right] \le \exp\bigg(-C_{\rm DP} \cdot \Big(\frac{n }{H(v)} \land \Big(\frac{n }{H_{\rm DP}^{(\varepsilon,\delta)}(v) } \Big)^2 \Big) \bigg),   
\end{equation}
}
where $C_{\rm DP}>0$ is a universal constant, and
{\small
\begin{equation}
\label{eqn:upperbound_dp_alg}
\!\!H_{\rm DP}^{(\varepsilon, \delta)} (v)\!=\! \max_{k\in \mathcal{S}, i \in \mathcal{A}: i \neq i^*_k}  \frac{\sqrt{\log(\frac{1.25}{\delta})\gamma^{\rm DP}_{k,i} }}{\sqrt{\varepsilon} \Delta_{k,i}}, \;\; \text{and}\; 
\end{equation}
\begin{equation}
\;\gamma_{k,i}^{\rm DP} \!=\!  { \sum_{k' , i' ,j' \in \mathcal{A} :  N_{k',i',j'}> 0 }   \left(\frac{w^{(k,i,i^*_k)}_{k',i',j'}}{N_{k',i',j'}}\right)^2}.
\end{equation}
} 
Consequently, 
{\small 
\begin{equation}
\limsup_{n \rightarrow \infty }\frac{1}{n}\log \mathbb{E}^{\rm DP\mbox{-}RL\mbox{-}LOW}_v \left[R_n \right] \le -\frac{C_{\mathrm{DP}}}{H(v)}. \label{eqn:limit_DP}
\end{equation}
}
\end{theorem}
The limiting statement in Eqn.~\eqref{eqn:limit_DP} implies that {\sc DP-RL-LOW} has the same order of the exponential decay rate as its non-differentially privacy counterpart {\sc RL-LOW} when $n$ is sufficiently large; in particular, $n > {(H_{\rm DP}^{(\varepsilon,\delta)}(v))^2}/H(v)$ suffices to nullify the effect of the privacy requirement. In other words, in the sense of the exponent, privacy comes ``for free'' for sufficiently large offline datasets. 

In addition, we derive the worst-case upper bound of {\sc DP-RL-LOW} in Appendix~\ref{app:DP-RL-Low-app}, which yields the form of  $O(\frac{1}{\sqrt{n}}+\frac{\sqrt{\log(1/\delta)}}{\varepsilon n})$ for the dependency in $n$, $ \varepsilon$ and $\delta$. This again implies that privacy comes for free for sufficiently large offline dataset.

\section{Related Works Beyond PbRL}
\textbf{Offline RL without Human Feedback} The domain of offline RL has been extensively researched over an extended period. Here, We focus on the recent works.~\citet{chen2019information} revisits and provides theoretical insights into the essential but underexplored assumptions of mild distribution shift and strong representation conditions in value-function approximations, advancing their necessity and applicability.
~\citet{xie2021policy} bridges the gap between online and offline reinforcement learning by introducing the policy finetuning problem, proposing algorithms that leverage a reference policy close to the optimal policy to achieve sample-efficient learning in episodic MDPs.
 ~\citet{yin2022near} investigates the statistical limits of offline reinforcement learning using linear models, introducing the variance-aware pessimistic value iteration method to improve learning bounds with offline data. More recently,~\citet{wang2022gap} enhances the understanding of gap-dependent sample complexity in offline reinforcement learning, demonstrating improved rates under specific policy coverage conditions and providing algorithms nearly matching lower bounds. Similarly, ~\citet{nguyen2023instance} investigated gap-dependent analysis for offline RL, providing both gap-dependent upper and lower bounds for performance with linear function approximation. 
 
 Overall, our study identifies a significant oversight in previous research: the absence of human feedback consideration. Consequently, our work represents the inaugural investigation into instance-dependent bounds within the context of offline reinforcement learning incorporating human feedback with pairwise comparisons.

\textbf{Dueling Bandits} 
The Dueling Bandits problem was first introduced by \citet{yue2009interactively}, sparking a substantial body of subsequent research on the topic. In this section, we highlight some relevant works. Inspired by the classical contextual bandits problem,~\citet{dudik2015contextual} extend the framework of duel bandits into a contextual setting, and they propose a new concept of von Neumann winner, a game-theoretic solution concept that addresses limitations of the Condorcet winner, along with three efficient algorithms for its online learning and approximation from data. 
In contrast, ~\citet{saha2021optimal} explore a distinct aspect of contextual dueling bandits through their proposed Subsetwise-Preference Feedback Model,
and the author presents two algorithms for pairwise preferences, achieving near-optimal regret bounds, and extending the analysis to general subsetwise preferences, demonstrating that the fundamental performance limits remain consistent regardless of the subset size.
However, this study mainly focuses on the worst-case analysis. More recently, ~\citet{di2024nearly} addressed the contextual dueling bandits with adversarial feedback, proposing a robust algorithm using uncertainty-weighted maximum likelihood estimation. Nonetheless, 
this work focuses on the adversarial setting, whereas our work examines the stochastic setting.

\section{Concluding Remarks}
This paper studies offline RLHF with pairwise comparisons, aiming to minimize simple regret by identifying the optimal action per state. We propose new algorithms achieving simple regret of the form $\exp(-\Omega(n/H(v)))$, where $n$ is the sample size and $H(v)$ captures instance-dependent hardness from suboptimality gaps. We also establish the first instance-dependent lower bound for this setting, matching our upper bound and proving exponential-rate optimality. To ensure privacy, we adapt our method to be $(\varepsilon, \delta)$-differentially private, showing that $H(v)$ remains asymptotically unchanged as $n \to \infty$. By establishing instance-dependent bounds of
exponential convergence, our results close a gap in prior works that focused on worst-case regret.


\section*{Acknowledgements}
 This research work was funded by two Singapore Ministry of Education Academic Research Fund Tier 1 grants (A-8000980-00-00 and A-8002934-00-00).
\bibliography{ref}

@article{ouyang2022training,
  title={Training language models to follow instructions with human feedback},
  author={Ouyang, Long and Wu, Jeffrey and Jiang, Xu and Almeida, Diogo and Wainwright, Carroll and Mishkin, Pamela and Zhang, Chong and Agarwal, Sandhini and Slama, Katarina and Ray, Alex and others},
  journal={Advances in neural information processing systems},
  volume={35},
  pages={27730--27744},
  year={2022}
}

@article{christiano2017deep,
  title={Deep reinforcement learning from human preferences},
  author={Christiano, Paul F and Leike, Jan and Brown, Tom and Martic, Miljan and Legg, Shane and Amodei, Dario},
  journal={Advances in neural information processing systems},
  volume={30},
  year={2017}
}

@article{ziegler2019fine,
  title={Fine-tuning language models from human preferences},
  author={Ziegler, Daniel M and Stiennon, Nisan and Wu, Jeffrey and Brown, Tom B and Radford, Alec and Amodei, Dario and Christiano, Paul and Irving, Geoffrey},
  journal={arXiv preprint arXiv:1909.08593},
  year={2019}
}

@article{stiennon2020learning,
  title={Learning to summarize with human feedback},
  author={Stiennon, Nisan and Ouyang, Long and Wu, Jeffrey and Ziegler, Daniel and Lowe, Ryan and Voss, Chelsea and Radford, Alec and Amodei, Dario and Christiano, Paul F},
  journal={Advances in Neural Information Processing Systems},
  volume={33},
  pages={3008--3021},
  year={2020}
}

@article{wu2021recursively,
  title={Recursively summarizing books with human feedback},
  author={Wu, Jeff and Ouyang, Long and Ziegler, Daniel M and Stiennon, Nisan and Lowe, Ryan and Leike, Jan and Christiano, Paul},
  journal={arXiv preprint arXiv:2109.10862},
  year={2021}
}

@article{nakano2021webgpt,
  title={Webgpt: Browser-assisted question-answering with human feedback},
  author={Nakano, Reiichiro and Hilton, Jacob and Balaji, Suchir and Wu, Jeff and Ouyang, Long and Kim, Christina and Hesse, Christopher and Jain, Shantanu and Kosaraju, Vineet and Saunders, William and others},
  journal={arXiv preprint arXiv:2112.09332},
  year={2021}
}

@article{glaese2022improving,
  title={Improving alignment of dialogue agents via targeted human judgements},
  author={Glaese, Amelia and McAleese, Nat and Tr{\k{e}}bacz, Maja and Aslanides, John and Firoiu, Vlad and Ewalds, Timo and Rauh, Maribeth and Weidinger, Laura and Chadwick, Martin and Thacker, Phoebe and others},
  journal={arXiv preprint arXiv:2209.14375},
  year={2022}
}

@article{bai2022training,
  title={Training a helpful and harmless assistant with reinforcement learning from human feedback},
  author={Bai, Yuntao and Jones, Andy and Ndousse, Kamal and Askell, Amanda and Chen, Anna and DasSarma, Nova and Drain, Dawn and Fort, Stanislav and Ganguli, Deep and Henighan, Tom and others},
  journal={arXiv preprint arXiv:2204.05862},
  year={2022}
}

@article{liu2023chain,
  title={Chain of hindsight aligns language models with feedback},
  author={Liu, Hao and Sferrazza, Carmelo and Abbeel, Pieter},
  journal={arXiv preprint arXiv:2302.02676},
  year={2023}
}

@inproceedings{zhu2023principled,
  title={Principled reinforcement learning with human feedback from pairwise or k-wise comparisons},
  author={Zhu, Banghua and Jordan, Michael and Jiao, Jiantao},
  booktitle={International Conference on Machine Learning},
  pages={43037--43067},
  year={2023},
  organization={PMLR}
}

@inproceedings{
zhan2024provable,
title={Provable Offline Preference-Based Reinforcement Learning},
author={Wenhao Zhan and Masatoshi Uehara and Nathan Kallus and Jason D. Lee and Wen Sun},
booktitle={The Twelfth International Conference on Learning Representations},
year={2024},
url={https://openreview.net/forum?id=tVMPfEGT2w}
}

@inproceedings{chowdhury2024differentially,
  title={Differentially private reward estimation with preference feedback},
  author={Chowdhury, Sayak Ray and Zhou, Xingyu and Natarajan, Nagarajan},
  booktitle={International Conference on Artificial Intelligence and Statistics},
  pages={4843--4851},
  year={2024},
  organization={PMLR}
}

@article{dwork2014algorithmic,
  title={The algorithmic foundations of differential privacy},
  author={Dwork, Cynthia and Roth, Aaron and others},
  journal={Foundations and Trends{\textregistered} in Theoretical Computer Science},
  volume={9},
  number={3--4},
  pages={211--407},
  year={2014},
  publisher={Now Publishers, Inc.}
}

@inproceedings{chaudhuri2011sample,
  title={Sample complexity bounds for differentially private learning},
  author={Chaudhuri, Kamalika and Hsu, Daniel},
  booktitle={Proceedings of the 24th Annual Conference on Learning Theory},
  pages={155--186},
  year={2011},
  organization={JMLR Workshop and Conference Proceedings}
}

@inproceedings{beimel2013private,
  title={Private learning and sanitization: Pure vs. approximate differential privacy},
  author={Beimel, Amos and Nissim, Kobbi and Stemmer, Uri},
  booktitle={International Workshop on Approximation Algorithms for Combinatorial Optimization},
  pages={363--378},
  year={2013},
  organization={Springer}
}

@article{ghazi2021deep,
  title={Deep learning with label differential privacy},
  author={Ghazi, Badih and Golowich, Noah and Kumar, Ravi and Manurangsi, Pasin and Zhang, Chiyuan},
  journal={Advances in neural information processing systems},
  volume={34},
  pages={27131--27145},
  year={2021}
}

@article{
shin2023benchmarks,
title={Benchmarks and Algorithms for Offline Preference-Based Reward Learning},
author={Daniel Shin and Anca Dragan and Daniel S. Brown},
journal={Transactions on Machine Learning Research},
issn={2835-8856},
year={2023},
url={https://openreview.net/forum?id=TGuXXlbKsn},
note={}
}

@inproceedings{
kim2023preference,
title={Preference Transformer: Modeling Human Preferences using Transformers for {RL}},
author={Changyeon Kim and Jongjin Park and Jinwoo Shin and Honglak Lee and Pieter Abbeel and Kimin Lee},
booktitle={The Eleventh International Conference on Learning Representations },
year={2023},
url={https://openreview.net/forum?id=Peot1SFDX0}
}

@book{sutton2018reinforcement,
  title={Reinforcement learning: An introduction},
  author={Sutton, Richard S and Barto, Andrew G},
  year={2018},
  publisher={MIT press}
}

@inproceedings{chen2022human,
  title={Human-in-the-loop: Provably efficient preference-based reinforcement learning with general function approximation},
  author={Chen, Xiaoyu and Zhong, Han and Yang, Zhuoran and Wang, Zhaoran and Wang, Liwei},
  booktitle={International Conference on Machine Learning},
  pages={3773--3793},
  year={2022},
  organization={PMLR}
}

@article{ibarz2018reward,
  title={Reward learning from human preferences and demonstrations in atari},
  author={Ibarz, Borja and Leike, Jan and Pohlen, Tobias and Irving, Geoffrey and Legg, Shane and Amodei, Dario},
  journal={Advances in neural information processing systems},
  volume={31},
  year={2018}
}

@article{palan2019learning,
  title={Learning reward functions by integrating human demonstrations and preferences},
  author={Palan, Malayandi and Landolfi, Nicholas C and Shevchuk, Gleb and Sadigh, Dorsa},
  journal={arXiv preprint arXiv:1906.08928},
  year={2019}
}

@book{vershynin2018high,
  title={High-dimensional probability: An introduction with applications in data science},
  author={Vershynin, Roman},
  volume={47},
  year={2018},
  publisher={Cambridge university press}
}

@book{meyer2000matrix,
  title={Matrix analysis and applied linear algebra},
  author={Meyer, Carl D},
  volume={71},
  year={2000},
  publisher={SIAM}
}

@book{tsybakov2009introduction,
  title={Introduction to Nonparametric Estimation},
  author={Tsybakov, A.B.},
  isbn={9780387790510},
  lccn={2008939894},
  series={Springer Series in Statistics},
  url={https://books.google.com.sg/books?id=JwzYNAEACAAJ},
  year={2009},
  publisher={Springer}
}

@article{qiao2024offline,
  title={Offline reinforcement learning with differential privacy},
  author={Qiao, Dan and Wang, Yu-Xiang},
  journal={Advances in Neural Information Processing Systems},
  volume={36},
  year={2024}
}

@inproceedings{chen2019information,
  title={Information-theoretic considerations in batch reinforcement learning},
  author={Chen, Jinglin and Jiang, Nan},
  booktitle={International Conference on Machine Learning},
  pages={1042--1051},
  year={2019},
  organization={PMLR}
}

@article{xie2021policy,
  title={Policy finetuning: Bridging sample-efficient offline and online reinforcement learning},
  author={Xie, Tengyang and Jiang, Nan and Wang, Huan and Xiong, Caiming and Bai, Yu},
  journal={Advances in neural information processing systems},
  volume={34},
  pages={27395--27407},
  year={2021}
}

@inproceedings{yin2022near,
  title={Near-optimal Offline Reinforcement Learning with Linear Representation: Leveraging Variance Information with Pessimism},
  author={Yin, Ming and Duan, Yaqi and Wang, Mengdi and Wang, Yu-Xiang},
  booktitle={International Conference on Learning Representation},
  year={2022}
}

@article{wang2022gap,
  title={On gap-dependent bounds for offline reinforcement learning},
  author={Wang, Xinqi and Cui, Qiwen and Du, Simon S},
  journal={Advances in Neural Information Processing Systems},
  volume={35},
  pages={14865--14877},
  year={2022}
}

@inproceedings{nguyen2023instance,
  title={On instance-dependent bounds for offline reinforcement learning with linear function approximation},
  author={Nguyen-Tang, Thanh and Yin, Ming and Gupta, Sunil and Venkatesh, Svetha and Arora, Raman},
  booktitle={Proceedings of the AAAI Conference on Artificial Intelligence},
  volume={37},
  number={8},
  pages={9310--9318},
  year={2023}
}

@inproceedings{yue2009interactively,
  title={Interactively optimizing information retrieval systems as a dueling bandits problem},
  author={Yue, Yisong and Joachims, Thorsten},
  booktitle={Proceedings of the 26th Annual International Conference on Machine Learning},
  pages={1201--1208},
  year={2009}
}

@inproceedings{dudik2015contextual,
  title={Contextual dueling bandits},
  author={Dud{\'\i}k, Miroslav and Hofmann, Katja and Schapire, Robert E and Slivkins, Aleksandrs and Zoghi, Masrour},
  booktitle={Conference on Learning Theory},
  pages={563--587},
  year={2015},
  organization={PMLR}
}

@inproceedings{
saha2021optimal,
title={Optimal Algorithms for Stochastic Contextual Preference Bandits },
author={Aadirupa Saha},
booktitle={Advances in Neural Information Processing Systems},
editor={A. Beygelzimer and Y. Dauphin and P. Liang and J. Wortman Vaughan},
year={2021},
url={https://openreview.net/forum?id=1lCZrXJBpM}
}

@article{di2024nearly,
  title={Nearly optimal algorithms for contextual dueling bandits from adversarial feedback},
  author={Di, Qiwei and He, Jiafan and Gu, Quanquan},
  journal={arXiv preprint arXiv:2404.10776},
  year={2024}
}

@article{ahmadian2024back,
  title={Back to basics: Revisiting reinforce style optimization for learning from human feedback in llms},
  author={Ahmadian, Arash and Cremer, Chris and Gall{\'e}, Matthias and Fadaee, Marzieh and Kreutzer, Julia and Pietquin, Olivier and {\"U}st{\"u}n, Ahmet and Hooker, Sara},
  journal={arXiv preprint arXiv:2402.14740},
  year={2024}
}

@inproceedings{
liu2024provably,
title={Provably Mitigating Overoptimization in {RLHF}: Your {SFT} Loss is Implicitly an Adversarial Regularizer},
author={Zhihan Liu and Miao Lu and Shenao Zhang and Boyi Liu and Hongyi Guo and Yingxiang Yang and Jose Blanchet and Zhaoran Wang},
booktitle={The Thirty-eighth Annual Conference on Neural Information Processing Systems},
year={2024},
url={https://openreview.net/forum?id=2cQ3lPhkeO}
}

@article{cen2024value,
  title={Value-Incentivized Preference Optimization: A Unified Approach to Online and Offline RLHF},
  author={Cen, Shicong and Mei, Jincheng and Goshvadi, Katayoon and Dai, Hanjun and Yang, Tong and Yang, Sherry and Schuurmans, Dale and Chi, Yuejie and Dai, Bo},
  journal={arXiv preprint arXiv:2405.19320},
  year={2024}
}

@article{chen2024order,
  title={On the Exponential Convergence for Offline RLHF with Pairwise Comparisons},
  author={Chen, Zhirui and Tan, Vincent YF},
  journal={arXiv preprint arXiv:2406.12205},
  year={2025}
}

\newpage
\appendix
\onecolumn
\begin{center}
{\Large\bf Technical Appendix for \\
``On the Exponential Convergence for Offline RLHF with Pairwise Comparisons'' }
\end{center}
{
\color{black}

\section{Details on Motivation}
\label{app:motivation}
\subsection{A Motivational Example of Label-DP}
\label{sub:DP_motivation}
In the development of question-answering (QA) systems, a common approach for improving response quality involves engaging users in a labeling process where they are asked to provide preference labels. Specifically, users evaluate pairs of system-generated responses to a given question and indicate which response they prefer. This method, often referred to as {\it pairwise preference labeling}, is integral to training RLHF algorithms that aim to optimize the relevance and utility of answers provided by QA systems.

Given our understanding of the nature of this process, our research emphasizes the importance of protecting the confidentiality of user-submitted preference labels. Without any concerted attempt to protect privacy, these labels, which directly reflect individual opinions or biases toward specific types of responses, can potentially reveal sensitive information, e.g., their preferences for some specific political parties. Therefore, we augment our RL-LOW with a label-DP protection mechanism to mitigate the risk of privacy breaches from the labels.

\subsection{A Motivational Example of MDP with Known Transition Probabilities}
\label{sub:MDP_motivation}
The MDP setting of our Section~\ref{sec:MDP_EXT}, in which transition probabilities are assumed to be known, is inspired by the inference framework of most language models. In this framework, states correspond to prompts, and actions represent the subsequent response or token. For instance, consider the current state: "The capital of France is". If the action "Paris" is taken, the subsequent state is deterministically defined as "The capital of France is Paris". While the reward associated with this
transition may be unknown, the transitions themselves are deterministic and well-defined in this context. Therefore, the
assumption of knowing transition probabilities aligns with the behaviour of language models, where the next state is known
based on the current state and action, and this motivates the MDP setting of our  Section~\ref{sec:MDP_EXT}.

\section{Numerical Simulations}\label{sec:simulation}
In this section, we present our numerical simulations of our algorithm {\sc RL-LOW} and and its differentially private counterpart {\sc DP-RL-LOW}. We compare them to the state-of-the-art (non-private algorithm) {\sc Pessimistic MLE} developed by \citet{zhu2023principled}.  We conduct the experiments on a synthetic dataset. Specifically, we set the number of states $S=2$,  the number of actions $A=10$, the dimensionality of the data $d=5$, the unknown parameter vector $\theta=[1,1,1,1,1]^\top$,  and the state distribution $\rho=[0.4,0.6]$. The feature vector of each action is generated as follows: For the $i$-th action of state $k\in \{1,2\}$, we first uniformly generate a $d$-dimentional vector $\phi'(k,i)$ with all non-negative elements and $\lVert \phi'(k,i) \rVert_1 = 1$. Then, for each state $k\in \{1,2\}$, we set the feature vector of $i$-th action as $\phi(k,i)=\phi'(k,i)-0.01(i-1) \theta$. That is, in both state $1$ and $2$, the best action is the first action, and the suboptimality gap of the $i$-th action is $0.05i$. In addition, for both states $k\in \{1,2\}$ and $i < j $, we set $N_{k,i,j}=\frac{1}{A(A-1)}$, i.e., the proportions  of comparisons for this instance are uniform. In the simulation, we use  $\lceil N_{k,i,j}n \rceil$ as the number of samples in the  comparison between actions $i$ and $j$ under state $k$.

As for the hyperparameters of {\sc Pessimistic MLE}, we follow the default setting of~\citet[Section~3]{zhu2023principled}. In addition,  {\sc Pessimistic MLE} only works under the assumption that $\langle 1, \theta \rangle=0$~\citep[Assumption 2.1]{zhu2023principled}. Therefore, when running the experiments of {\sc Pessimistic MLE}, we further set $d=6$, $\theta=[1,1,1,1,1,-5]$ and the $6$-th element of each feature vector is set to $0$. Then, this new instance is mathematically equivalent to the original instance and additionally satisfies Assumption 2.1 of \citet{zhu2023principled} which is needed for   {\sc Pessimistic MLE}. 

The simulation results are shown in Figure~\ref{fig:fig1}. We run each experiment $200$ times, and report the average and standard deviation. From Figure~\ref{fig:fig1} (left), we observe that {\sc RL-LOW} is  inferior to {\sc Pessimistic MLE}  for small  $n$. However,  since {\sc RL-LOW} is instance-dependent  optimal in the exponential decay rate and in its dependence on the hardness parameter $H(v)$, the experimental findings depicted in Figure \ref{fig:fig1} (left) corroborate the empirical superiority of our proposed {\sc RL-LOW} algorithm over {\sc Pessimistic MLE} for $n$ sufficiently large ($n >150$ suffices). This observation underscores the efficacy of our novel algorithmic design based on locally optimal weights. Furthermore, from  Figure~\ref{fig:fig1} (left), we also observe that as the sample size $n$ increases, the performance of {\sc DP-RL-LOW} converges to that of {\sc RL-LOW}, consistent with our theoretical findings in Theorem~\ref{thm:dp_upperbound}.

Lastly, as shown in Figure \ref{fig:fig1} (right) and the curve of {\sc DP-RL-LOW} of Figure \ref{fig:fig1} (left), it is evident that achieving comparable performance between \textsc{RL-LOW} and \textsc{DP-RL-LOW} may necessitate substantially larger sample sizes $n$ when considering small privacy parameters of $\varepsilon$ and $\delta$. This observation is again consistent with our theoretical findings in Theorem~\ref{thm:dp_upperbound}.

\begin{figure*}[!ht]
\centering
    \begin{subfigure}{0.48\textwidth}
        \centering
        \includegraphics[width=\textwidth]{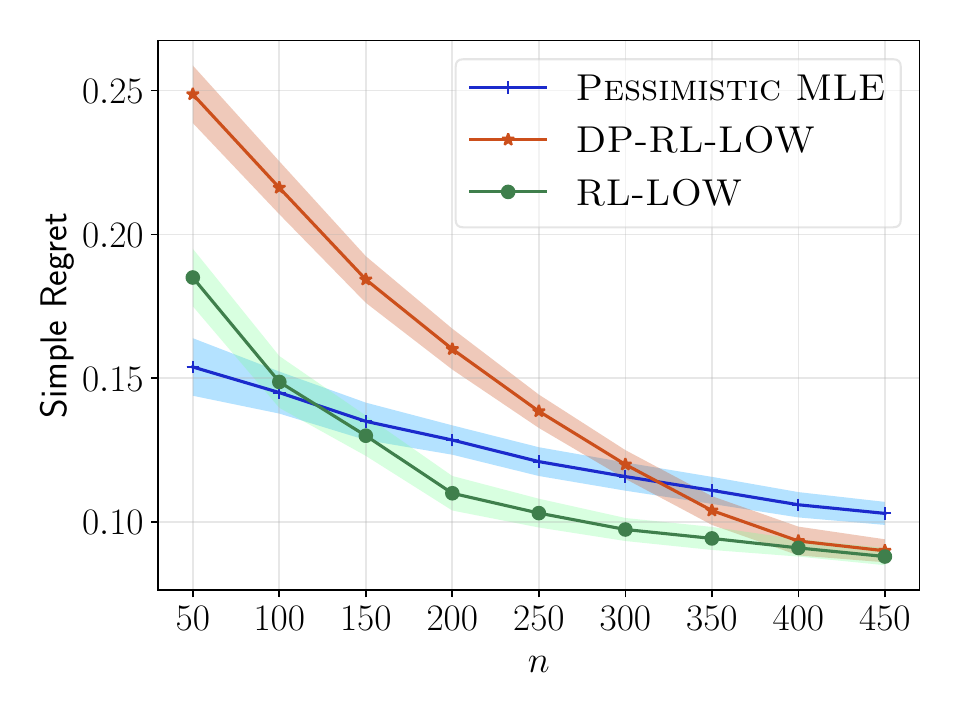}
    \end{subfigure}
    \begin{subfigure}{0.48\textwidth}
        \centering
        \includegraphics[width=\textwidth]{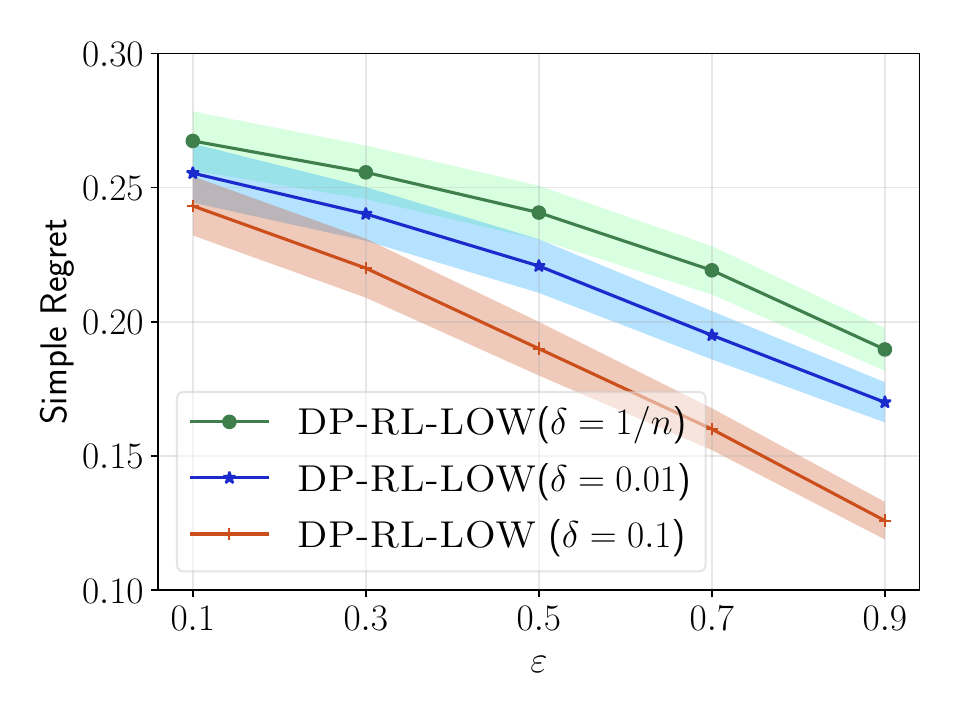}
    \end{subfigure}

    \caption{ Comparison of {\sc RL-LOW}  and  {\sc DP-RL-LOW}  to  {\sc Pessimistic MLE} on average simple regret and standard deviation (shaded area). In the left figure, we set $\delta=0.2$ and $\varepsilon=0.9$ for {\sc DP-RL-LOW}. In the right figure, we set $n=400$ for all policies.} 
    \label{fig:fig1}  
\end{figure*}

\section{More Explanations for the Hardness  Parameter $H(v)$}
\label{sub:hardness_intuition}
The hardness parameter $H(v)$ is inversely proportional to the square of the suboptimality gaps across various states and actions. Specifically, a smaller suboptimality gap implies an increased hardness parameter, indicating that the instance is more challenging to learn. This relationship underscores the significance of the suboptimality gap as a critical measure in evaluating the complexity and learning difficulty of each instance. 

For example, let's suppose $S=1, A=2$, $r_{1,1}=0$ and $r_{1,2}=1$. Given offline data with of size $n$ (i.e., these are $n$ history records for "action $1$ wins action $2$" or "action $2$ wins action $1$"), if the learner picks the action with the most winning records (as it does in our algorithm RL-LOW), then by Hoeffding's inequality, it will suffer an upper bound of expected simple regret of $\exp(-C \cdot \frac{n}{ {(r_{1,2}-r_{1,1})^{-2}}})$ for a constant $C$ that does not depend on $r_{1,2}-r_{1,1}$ (in fact this upper bound is also tight in the hardness parameter according to our lower bound). Notice the exponent is $\Theta(\frac{-n}{ {(r_{1,2}-r_{1,1})^{-2}}})$, and the hardness parameter $H(v)$ is exactly $ (r_{1,2}-r_{1,1})^{-2}$ under this instance. 

}
\section{Useful facts}
In this section, we collate some useful facts that will be used in the subsequent proofs.

\begin{definition}[SubGaussian norm]
A random variable $X$ is  {\em subGassian} if it has a finite subGaussian norm denoted as $\|X\|_{\psi_2}$ and defined as
$$
\|X\|_{\psi_2}=\inf \left\{c>0: \mathbb{E}\left[\exp \left(\frac{X^2}{c^2}\right)\right] \leq 2\right\} < + \infty.
$$
\end{definition}

\begin{definition}[Variance proxy]
The variance proxy of a subGaussian random variable $X$ is denoted as $\|X\|_{\rm vp}^2$ and defined as 
$$\|X\|_{\rm vp}^2 \coloneqq \inf \left\{ s^2>0 : \mathbb{E}\left[\exp({(X-\mathbb{E}[X]) t})\right] \leq e^{\frac{s^2 t^2}{2}}, \;\; \forall\, t>0 \right\}.$$ 
\end{definition}
\begin{lemma}[Linear combination of subGaussian random variables]
\label{lemma:sub-Gaussian-combination}
Let $X_1, \ldots, X_n$ be independent subGaussian random variables, where the variance proxy of $X_i$ is $\sigma_i^2$ . Then, for any $a_1,\dots,a_n \in \mathbb{R}$, the random variable $\sum_{i=1}^n a_i X_i$ is a subGaussian random variable with variance proxy $\sigma^2=\sum_{i=1}^n a_i^2 \sigma_i^2$.
\end{lemma}

\begin{lemma}{(Tail bound of subGassian random variables)} 
\label{lemma:tail_bound_of_subgaussian}
Suppose $X$ is subGaussian with variance proxy $\sigma^2$. Then, for any $\epsilon>0$, we have
$$
\operatorname{Pr}(X-\mathbb{E}[X] \geq \epsilon) \leq \exp \left(-\epsilon^2 /\left(2 \sigma^2\right)\right),
$$ and
$$
\operatorname{Pr}(X-\mathbb{E}[X] \le -\epsilon) \leq \exp \left(-\epsilon^2 /\left(2 \sigma^2\right)\right),
$$
\end{lemma}

\begin{lemma}{(Adapted from~\citet[Proposition 2.5.2]{vershynin2018high})}
\label{lemma:relation_phi2_and_vp}
For any subGaussian random variable $X$,
$$
\|X\|_{\rm vp} \le C \|X\|_{\psi_2},
$$
where $C \le 6\sqrt{2e} \cdot (3\sqrt{\log 2}+1)$.
If $\mathbb{E}[ X  ]=0$, then we have
$$
\|X\|_{\psi_2} \le \sqrt{6} \|X\|_{\rm vp}. 
$$
\end{lemma}

\section{Proof of Proposition~\ref{prop:imposs_results}}

\begin{lemma}
\label{lemma:existence_of_pair_invole_best_action}
For any inconsistent instance $v=  (\rho,\mathcal{S},\mathcal{A},\phi,N, \theta)$, there exists $(k,i)\in \mathcal{S}\times \mathcal{A}$ with $i \neq i^*_k$ such that 
$$
\phi(k,i) - \phi(k,i^*_k) \notin {\rm Span} \{\phi(k',i')-\phi(k',j') \given N_{k',i',j'} >0, (k',i',j') \in \mathcal{S}\times \mathcal{A}^2 \}.
$$
\end{lemma}
\begin{proof}
 We  prove this result  by contradiction. Fix any inconsistent instance $v= (\rho,\mathcal{S},\mathcal{A},\phi,N, \theta)$. Assume that  for all $(k,i)\in \mathcal{S}\times \mathcal{A}$ with $i \neq i^*_k$, it holds 
\begin{equation}
\label{eq:contradition1}
\phi(k,i) - \phi(k,i^*_k) \in {\rm Span} \{\phi(k',i')-\phi(k',j') \given N_{k,i,j} >0, (k',i',j') \in \mathcal{S}\times \mathcal{A}^2 \}.
\end{equation}
By the fact that $\phi(k,i) - \phi(k,j) = (\phi(k,i) - \phi(k,i^*_k)) - (\phi(k,j) - \phi(k,i^*_k)))$, Eqn.~\eqref{eq:contradition1} implies that for any $(k,i,j)\in \mathcal{S}\times \mathcal{A} \times {A}$ with $i \neq j$, it holds 
\begin{equation}
\label{eq:contradition2}
\phi(k,i) - \phi(k,j) \in {\rm Span} \{\phi(k',i')-\phi(k',j') \given N_{k,i,j} >0, (k',i',j') \in \mathcal{S}\times \mathcal{A}^2 \},
\end{equation}
which is a contradiction to that $v$ is inconsistent. This completes the proof of Lemma~\ref{lemma:existence_of_pair_invole_best_action}.
\end{proof}

\begin{proof}[Proof of Proposition~\ref{prop:imposs_results}]
Fix any inconsistent instance $v=  (\rho,\mathcal{S},\mathcal{A},\phi,N, \theta)$. By Lemma~\ref{lemma:existence_of_pair_invole_best_action}, there exists $(\underline{k},\underline{i})\in \mathcal{S}\times \mathcal{A}$ with $\underline{i} \neq i^*_{\underline{k}}$ such that 
$$
\phi(\underline{k},\underline{i}) - \phi(\underline{k},i^*_{\underline{k}}) \notin {\rm Span} \{\phi(k',i')-\phi(k',j') \given N_{k',i',j'} >0, (k',i',j') \in \mathcal{S}\times \mathcal{A}^2 \}.
$$
That is, there exists a vector $\bold{z} \in \mathbb{R}^d$ such that 
$$
\langle \bold{z},  \phi(\underline{k},\underline{i}) - \phi(\underline{k},i^*_{\underline{k}}) \rangle = -2 \langle \theta,  \phi(\underline{k},\underline{i}) - \phi(\underline{k},i^*_{\underline{k}}) \rangle
$$
and for all $(k',i',j') \in \mathcal{S}\times \mathcal{A}^2$ with  $N_{k',i',j'} >0$, 
$$
\langle \bold{z},  \phi(k',i') - \phi(k',j') \rangle = 0.
$$

Finally, we let $\theta' =  \theta + \bold{z}$, and instance $v' =  (\rho,\mathcal{S},\mathcal{A},\phi,N, \theta')$.
By the fact that 
$ \langle \theta,  \phi(k',i') - \phi(k',j') \rangle = \langle \theta',  \phi(k',i') - \phi(k',j') \rangle  $ for all $(k',i',j') \in \mathcal{S}\times \mathcal{A}^2$ with  $N_{k',i',j'} >0$ we get that for all $n\ge 1$
\begin{equation*}
    D_{\rm KL}(P_v^{(n)}, P_{v'}^{(n)}) =0.
\end{equation*}
Therefore, we get that $P_v^{(n)}$ is equal to $ P_{v'}^{(n)}$. 
In addition, by definition of $R_n$, we get that 
\begin{align}
&\mathbb{E}^{\pi}_v [R_n] +\mathbb{E}^{\pi}_{v'} [R_n] \nonumber \\
& = \mathbb{E}_v^\pi \left[ \sum_{k \in \mathcal{S}} \rho_{k}   \left( \max_{j\in \mathcal{A} } \langle \phi(k,j) - \phi(k,\hat{i}_k), \theta \rangle  \right) \right]+ 
 \mathbb{E}_{v'}^\pi \left[  \sum_{k \in \mathcal{S}} \rho_k \left(\max_{j\in \mathcal{A} } \langle \phi(k,j) - \phi(k,\hat{i}_k), \theta' \rangle \right)  \right] \nonumber \\
& \stackrel{(a)}{=} \mathbb{E}_v^\pi \left[ \sum_{k \in \mathcal{S}} \rho_{k}   \left( \max_{j\in \mathcal{A} } \langle \phi(k,j) - \phi(k,\hat{i}_k), \theta \rangle  + \max_{j\in \mathcal{A} } \langle \phi(k,j) - \phi(k,\hat{i}_k), \theta' \rangle \right) \right] \\
& \ge \rho_{\underline{k}} \min_{i \in \mathcal{A}} \left[ \max_{j\in \mathcal{A} } \langle \phi(\underline{k},j) - \phi(\underline{k},i), \theta \rangle  + \max_{j\in \mathcal{A} } \langle \phi(\underline{k},j) - \phi(\underline{k},i), \theta' \rangle \right],  \nonumber
\end{align}
where (a) follows from the fact that  $P_v^{(n)}$ is equivalent with $ P_{v'}^{(n)}$.

Further, by the definition of $v$ and $v'$, we get that $$\min_{i \in \mathcal{A}} \left[ \max_{j\in \mathcal{A} } \langle \phi(\underline{k},j) - \phi(\underline{k},i), \theta \rangle  + \max_{j\in \mathcal{A} } \langle \phi(\underline{k},j) - \phi(\underline{k},i), \theta' \rangle \right]>0,$$ and recall that $\rho_{\underline{k}}>0$.  This completes the proof of Proposition~\ref{prop:imposs_results}

\end{proof}

\section{Proof of Proposition~\ref{propo:transitivity} and More Details on Computational Complexity}
\label{sec:existing_i_hat}
\subsection{Proof of Proposition~\ref{propo:transitivity}}
We first provide  analytical solutions of  $w^{(k,i,j)}$ and $\gamma_{k,i,i^*_{k}}$.
\begin{lemma}
\label{lemma:analytic_solutions_of_opt_weights}
Fix any consistent instance $v=  (\rho,\mathcal{S},\mathcal{A},\phi,N, \theta)$. Recall the definitions of $w^{(k,i,j)}$ and $\gamma_{k,i,i^*_{k}}$ in Eqn.~\eqref{eqn:opt_weights} and Eqn.~\eqref{eq:hardness}, respectively. Then, for any $(k,i)\in \mathcal{S} \times \mathcal{A}$ with $i \neq i^*_{k}$,
\begin{equation}
    \gamma_{k,i,i^*_k} = \lVert [\phi(k,i^*_k - \phi(k,i)]_\mathcal{G} \rVert^2_{V^{-1}}
\end{equation}
and for any $(k,i,j)\in \mathcal{S} \times \mathcal{A}^2$ with $i \neq j$, 
\begin{equation}
    w^{(k,i,j)}_{k',i',j'} =  N_{{k'},{i'},{j'}}  [\phi({k'},{i'})-\phi({k'},{j'})]_{\mathcal{G}}^\top  V^{-1} [\phi(k,i)-\phi(k,j)]_{\mathcal{G}}
\end{equation}
where $V$ and $\mathcal{G}$ are as defined in Lemma~\ref{lemma:min_distance_main}.

\end{lemma}

\begin{proof}[Proof of Lemma~\ref{lemma:analytic_solutions_of_opt_weights}]
Fix any $(k,i,j)\in \mathcal{S} \times \mathcal{A}^2$ with $i \neq j$.
By definition, the optimization problem of Eqn.~\eqref{eqn:opt_weights} is equivalent to 

\begin{equation}
\label{eq:equivalent_local_optimal}
\min_{u \in \mathbb{R}^{S\times A \times A}  } \sum_{k' \in \mathcal{S},i',j'\in \mathcal{A} :N_{k',i',j'}> 0} \frac{  (u_{k',i',j'})^2 }{N_{k',i',j'}} 
\end{equation} 
subject to
\begin{equation*}
[\phi(k,i)-\phi(k,j)]_{\mathcal{G}}= \sum_{k'\in \mathcal{S}, i',j'\in \mathcal{A}: N_{k',i',j'}>0} u_{k',i',j'}[\phi(k',i')-\phi(k',j')]_{\mathcal{G}}.
\end{equation*}
The Lagrangian of the above constrained optimization problem is 
\begin{align}
&L(u,\lambda) = \sum_{k' \in \mathcal{S},i',j'\in \mathcal{A} :N_{k',i',j'}> 0} \frac{  (u_{k',i',j'})^2 }{N_{k',i',j'}} + \lambda^\top \bigg([\phi(k,i)-\phi(k,j)]_{\mathcal{G}} \nonumber \\
   & \hspace{3cm} - \sum_{k'\in \mathcal{S}, i',j'\in \mathcal{A}: N_{k',i',j'}>0} u_{k',i',j'}[\phi(k',i')-\phi(k',j')]_{\mathcal{G}}\bigg), \nonumber
\end{align}
for $u\in \mathbb{R}^{S \times A \times A}$ and $\lambda \in \mathbb{R}^{\lvert \mathcal{G} \rvert}$.   
Then, by solving $\frac{\mathrm{d} L}{\mathrm{d} \lambda} = 0$ and $\frac{\mathrm{d} L}{\mathrm{d} u_{k',i',j'}} = 0$ for all $(k',i',j')\in \mathcal{S} \times \mathcal{A}^2$ with $N_{k',i',j'}>0$, we obtain that the minimum of~\eqref{eq:equivalent_local_optimal} is
\begin{equation}
\lambda=-2 V^{-1} [\phi(k,i)-\phi(k,j)]_{\mathcal{G}} \nonumber \\
\end{equation}
and
\begin{equation}
    u_{k',i',j'} =N_{{k'},{i'},{j'}}  [\phi({k'},{i'})-\phi({k'},{j'})]_{\mathcal{G}}^\top  V^{-1} [\phi(k,i)-\phi(k,j)]_{\mathcal{G}}
\end{equation}
for all $(k',i',j') \in \mathcal{S}\times \mathcal{A}^2$ with $N_{k',i',j'}>0$.
That is, for any $(k,i)\in \mathcal{S} \times \mathcal{A}$ with $i \neq i^*_{k}$,
\begin{equation}
    \gamma_{k,i,i^*_k} = \lVert [\phi(k,i^*_k) - \phi(k,i)]_\mathcal{G} \rVert^2_{V^{-1}}
\end{equation}
This completes the desired proof.
\end{proof}

Then, we are ready to prove Proposition~\ref{propo:transitivity}
\begin{proof}[Proof of Propostion~\ref{propo:transitivity}]
By Lemma~\ref{lemma:analytic_solutions_of_opt_weights}, we get that under {\sc RL-LOW}, for any $(k,i,j,j_2)\in \mathcal{S} \times \mathcal{A}^3$, 
\begin{align}
    &\hat{r}_{k,i,j} +\hat{r}_{k,j,j_2} \nonumber \\
    &=  \sum_{(k',i',j')\in \mathcal{S}\times \mathcal{A}^2} N_{{k'},{i'},{j'}}  [\phi({k'},{i'})-\phi({k'},{j'})]_{\mathcal{G}}^\top  V^{-1} [\phi(k,i)-\phi(k,j)]_{\mathcal{G}} \log\Big(\frac{B^{\rm CLP}_{k',i',j'}}{1-B^{\rm CLP}_{k',i',j'}} \Big) \nonumber \\
    & \hspace{1cm} +  \sum_{(k',i',j')\in \mathcal{S}\times \mathcal{A}^2} N_{{k'},{i'},{j'}}  [\phi({k'},{i'})-\phi({k'},{j'})]_{\mathcal{G}}^\top  V^{-1} [\phi(k,j)-\phi(k,j_2)]_{\mathcal{G}} \log\Big(\frac{B^{\rm CLP}_{k',i',j'}}{1-B^{\rm CLP}_{k',i',j'}} \Big) \nonumber \\
 & = \sum_{(k',i',j')\in \mathcal{S}\times \mathcal{A}^2} N_{{k'},{i'},{j'}}  [\phi({k'},{i'})-\phi({k'},{j'})]_{\mathcal{G}}^\top  V^{-1} [\phi(k,i)-\phi(k,j_2)]_{\mathcal{G}} \log\Big(\frac{B^{\rm CLP}_{k',i',j'}}{1-B^{\rm CLP}_{k',i',j'}} \Big) \nonumber \\
 &=\hat{r}_{k,i,j_2},
\end{align}
which implies that  $\lvert \hat{\mathcal{I}}_k\rvert \ge 1$ and \begin{equation}
         \argmax_{i\in \mathcal{A}} \hat{r}_{k,i,j_1} = \argmax_{i\in \mathcal{A}} \hat{r}_{k,i,j_2} = \hat{\mathcal{I}}_k \quad \mbox{for any}\quad   j_1,j_2\in \mathcal{A}.
     \end{equation}
This completes the proof of  Proposition~\ref{propo:transitivity}.
\end{proof}

Following the same lines as the proof of  Propostion~\ref{propo:transitivity}, we get the corollary below.
\begin{corollary}
\label{propo:transitivity_dp}

     For any consistent instance $v$ and using estimate of the best action under each state $k$ as prescribed by {\sc DP-RL-LOW}, we have $\lvert \tilde{\mathcal{I}}_k\rvert \ge 1$ and \begin{equation}
         \argmax_{i\in \mathcal{A}} \tilde{r}_{k,i,j_1} = \argmax_{i\in \mathcal{A}} \tilde{r}_{k,i,j_2} = \tilde{\mathcal{I}}_k \quad \mbox{for any}\quad   j_1,j_2\in \mathcal{A}.
     \end{equation}

\end{corollary}
\subsection{Computational Complexity}
\subsubsection{computational complexity of {\sc RL-LOW} and {\sc DP-RL-LOW}}
\label{subsec:computatoinal_complexity}
With Proposition~\ref{propo:transitivity}, it is not necessary to compute all values of $\hat{r}_{k,i,j}$ for each $(k,i,j) \in \mathcal{S} \times \mathcal{A}^2$. Instead, we only need to compute  $\hat{r}_{k,i,j^\dag}$ for each $(k,i) \in \mathcal{S} \times \mathcal{A}$ where $j^\dag \in \mathcal{A}$ is (arbitrarily) fixed (e.g., $j^\dag=1$), and let $\hat{i}_k \in \arg\max_{i \in \mathcal{A}} \hat{r}_{k,i, j^\dag} $.

Note that by Lemma~\ref{lemma:analytic_solutions_of_opt_weights}, we have 
$$
\hat{r}_{k,i, j^\dag} =  \sum_{(k',i',j')\in \mathcal{S}\times \mathcal{A}^2} N_{{k'},{i'},{j'}}  [\phi({k'},{i'})-\phi({k'},{j'})]_{\mathcal{G}}^\top  V^{-1} [\phi(k,i)-\phi(k,j^\dag)]_{\mathcal{G}} \log\Big(\frac{B^{\rm CLP}_{k',i',j'}}{1-B^{\rm CLP}_{k',i',j'}} \Big).
$$

Then, if we pre-calculate a {\em global vector} of  $\sum_{(k',i',j')\in \mathcal{S}\times \mathcal{A}^2} N_{{k'},{i'},{j'}}  [\phi({k'},{i'})-\phi({k'},{j'})]_{\mathcal{G}}^\top  V^{-1}  \log\Big(\frac{B^{\rm CLP}_{k',i',j'}}{1-B^{\rm CLP}_{k',i',j'}} \Big)$, we can compute each $\hat{r}_{k,i, j^\dag}$ for $(k,i) \in \mathcal{S}\times \mathcal{A}$ in $\mathcal{O}(d)$ time complexity. Similarly, $\tilde{r}_{k,i,j^\dag}$ in {\sc DP-RL-LOW} can be computed through an analogous procedure.
Hence, the overall computational complexity of {\sc RL-LOW} and {\sc DP-RL-LOW} are $\mathcal{O}(SAd + nd^2 + d^3)$, where the term ``$SAd$'' corresponds to the natural process of scanning the feature vectors for all state-action pairs, and the terms ``$ nd^2 + d^3 $'' corresponds to compute $V^{-1}$ in above.

{\color{black} The computational complexity's dependence on ``\( SAd \)
'' is inevitable in our current framework, as the generation of the output \(\hat{i}_k\) is required for all \( k \in \mathcal{S} \).
Nonetheless, in an alternative setting in which the set of best actions to be estimated $\{ \hat{i}_k \in \mathcal{A} \}_{   k\in \mathcal{S} }$ is replaced by a parametric function $ \{ \hat{i} (k;\vartheta) \in \mathcal{A}\}$ where a parameter $\vartheta$ is to be estimated. In this setting, the previously global vector can be utilized to represent $\vartheta$. Then, in this setting, the overall computational complexity becomes $ O(nd^2 + d^3)$, which is required to compute the global vector.}

\subsubsection{computational complexity of {\sc RL-LOW-MDP}}
In {\sc RL-LOW-MDP}, after each $\hat{r}_{k,i, j^\dag}$ for $(k,i) \in \mathcal{S}\times \mathcal{A}$ is computed, the rest process is the standard RL problem under the condition that the transition and reward functions are known. Hence, the overall computational complexity of {\sc RL-LOW-MDP} is $\mathcal{O}(SAd + nd^2 + d^3 + g(S,A))$, where the term $g(S,A)$ corresponds to the above standard problem {\color{black} that can be solved by asynchronous dynamic programming or linear programming~\citep{sutton2018reinforcement}. Particularly $g(S,A)=O(\kappa SA)$ by using asynchronous dynamic programming,  and $\kappa$ is a hyperparameter that represents the average number of iteration steps, which controls the solution precision.}
\section{Proof of Lower Bound}

Before we prove the lower bound, we first give a useful corollary.
The following corollary is a direct result of Pinsker's inequality.
\begin{corollary}
\label{corollary:bern_kl}
Fix any $C\in (0, \frac{1}{2})$. For any $p,q \in (0,1)$ with $\min(p,1-p) \ge C$ and $\min(q,1-q) \ge C$, we have
$$
 2(p-q)^2 \le d_{\rm KL}(p,q) \le \frac{2}{C} (p-q)^2
$$
where $d_{\rm KL}(p,q)$ denotes the KL divergence between the Bernoulli distributions with parameters of $p$ and $q$.
\end{corollary}

Then, we prove Lemma~\ref{lemma:kl_divergence_between_instances} that reveals the KL divergence between instances. Recall that given any instance $v$, we let $P_v^{(n)}$ denote the distribution of $(\sigma_i)_{i=1}^n$. The following lemma gives an estimation of the KL divergence between instance $v$ and $v'$ that share the same parameters but $\theta$.

\begin{replemma}{lemma:kl_divergence_between_instances}
For any instance $v=  (\rho,\mathcal{S},\mathcal{A},\phi,N, \theta)$ and $v'=  (\rho,\mathcal{S},\mathcal{A},\phi,N, \theta')$, it holds that
\begin{align}
 2n\exp(-4R_{\max})
 \le  \frac{ D_{\rm KL}(P_v^{(n)} \| P_{v'}^{(n)})    }{\tilde{D}(v,v')} \le 2n\exp(2R_{\max} ) 
\end{align}
where 
\begin{align}
  \tilde{D}(v,v')=   \sum_{k \in \mathcal{S},i,j \in \mathcal{A}} N_{k,i,j} ( \langle \phi(k,i)-\phi(k,j), \theta-\theta' \rangle )^2 ,\label{eqn:defDtilde_repeat}
\end{align}
and where $R_{\max}=\max_{k \in \mathcal{S}, i\in \mathcal{A}} \max \{\lvert \langle \phi(k,i),\theta \rangle \rvert, \lvert \langle \phi(k,i),\theta' \rangle \rvert \}$ is the maximum absolute reward   in these two instances.
\end{replemma}
\begin{proof}
Fix any consistent instance $v=  (\rho,\mathcal{S},\mathcal{A},\phi,N, \theta)$. By the chain rule of the KL divergence, we have
\begin{align}
 & D_{\rm KL}(P_v^{(n)} \| P_{v'}^{(n)})     \nonumber \\
&= n \sum_{k \in \mathcal{S},i,j \in \mathcal{A}} N_{k,i,j} d_{\rm KL}\bigg(\frac{\exp(\langle \phi(k,i), \theta \rangle)}{\exp(\langle \phi(k,i), \theta\rangle) + \exp(\langle \phi(k,j), \theta \rangle)} ,   \nonumber\\*
&\qquad\qquad\qquad\qquad\qquad\qquad\qquad\qquad\frac{\exp(\langle \phi(k,i), \theta' \rangle)}{\exp(\langle \phi(k,i, \theta' \rangle) + \exp(\langle \phi(k,j), \theta' \rangle)}  \bigg) \nonumber \\*
& =  n \sum_{k \in \mathcal{S},i,j \in \mathcal{A}} N_{k,i,j} d_{\rm KL}({\rm Sig}(\langle \phi(k,i)-\phi(k,j), \theta \rangle), {\rm Sig}(\langle \phi(k,i)-\phi(k,j), \theta' \rangle)) \label{eq:to_combined_0}
\end{align}
where ${\rm Sig}(\cdot)$ represents the Sigmoid function.
 By the fact that $-2R_{\rm max} \le \langle \phi(k,i)-\phi(k,j), \theta \rangle \le 2R_{\rm max}$ and  $-2R_{\rm max} \le \langle \phi(k,i)-\phi(k,j), \theta' \rangle \le 2R_{\rm max}$ and that  $\frac{\mathrm{d}   {\rm Sig}(x)}{\mathrm{d}x}= \frac{\exp(-x)}{(\exp(-x)+1)^2}$, we further get that 
\begin{align}
&\lvert {\rm Sig}(\langle \phi(k,i)-\phi(k,j) , \theta \rangle) - {\rm Sig}(\langle \phi(k,i)-\phi(k,j), \theta' \rangle) \rvert\nonumber\\*
&\qquad\le \lvert  \langle \phi(k,i)-\phi(k,j), \theta - \theta' \rangle \rvert\label{eq:bound_sig1}
\end{align}
and 
\begin{align}
&\lvert {\rm Sig}(\langle \phi(k,i)-\phi(k,j) , \theta \rangle) - {\rm Sig}(\langle \phi(k,i)-\phi(k,j), \theta' \rangle) \rvert \nonumber\\*
&\qquad\ge \exp(-2R_{\rm max}) \lvert \langle \phi(k,i)-\phi(k,j), \theta - \theta' \rangle \rvert.\label{eq:bound_sig2}
\end{align}
Then, by Corollary~\ref{corollary:bern_kl}, we get that Eqn.~\eqref{eq:bound_sig1} and Eqn.~\eqref{eq:bound_sig2} imply that 
\begin{align}
&d_{\rm KL}({\rm Sig}(\langle \phi(k,i)-\phi(k,j), \theta \rangle), {\rm Sig}(\langle \phi(k,i)-\phi(k,j), \theta' \rangle)) \nonumber\\*
&\qquad \le 2\exp(2R_{\rm max}) \lvert \langle \phi(k,i)-\phi(k,j), \theta - \theta' \rangle \rvert^2\label{eq:to_combined_1}
\end{align}
and
\begin{align}
&d_{\rm KL}({\rm Sig}(\langle \phi(k,i)-\phi(k,j), \theta \rangle), {\rm Sig}(\langle \phi(k,i)-\phi(k,j), \theta' \rangle)) \nonumber\\*
&\qquad\ge 2\exp(-4R_{\rm max}) \lvert \langle \phi(k,i)-\phi(k,j), \theta - \theta' \rangle \rvert^2.\label{eq:to_combined_2}
\end{align}
Finally, combining Eqn.~\eqref{eq:to_combined_0}, Eqn.~\eqref{eq:to_combined_1} and Eqn.~\eqref{eq:to_combined_2}, we complete the proof of Lemma~\ref{lemma:kl_divergence_between_instances}.
\end{proof}

Then, recall that we denote $\tilde{D}(\cdot, \cdot)$ as the approximation of KL divergence between instance $v$ and $v'$ that share the same parameter except $\theta$, i.e., for any $v=(\rho,\mathcal{S},\mathcal{A},\phi,N, \theta)$ and  $v=(\rho,\mathcal{S},\mathcal{A},\phi,N, \theta')$,
$$
\tilde{D}(v,v') \coloneqq \sum_{k \in \mathcal{S},i,j \in \mathcal{A}} N_{k,i,j} ( \langle \phi(k,i)-\phi(k,j), \theta-\theta' \rangle )^2.
$$
In addition, recall that for any $\bold{z} \in \mathbb{R}^d$, $\eta \in \mathbb{R}$ and consistent instance $v=(\rho,\mathcal{S},\mathcal{A},\phi,N, \theta)$,  we denote ${\rm Alt}(v,\bold{z},\eta)$ as the set of instances that share the same parameter except $\theta$ and satisfy 
$$
  \langle \bold{z}, \theta'  -\theta \rangle  = \eta 
$$
for all $v' = (\rho,\mathcal{S},\mathcal{A},\phi,N, \theta') \in  {\rm Alt}(v,\bold{z},\eta) $.

We are ready to prove Lemma~\ref{lemma:min_distance_main} that reveals a useful property for ${\rm Alt}(\cdot)$ and $\tilde{D}(\cdot)$
 
\begin{replemma}{lemma:min_distance_main}
Let $\mathcal{G}$ be an arbitrary orthonormal basis of ${\rm Span} \{\phi(k',i')-\phi(k',j') : (k',i',j')\in \mathcal{S}\times \mathcal{A}^2 \text{ and } N_{k',i',j'}>0  \}$. Also let $[\mathbf{w}]_{\mathcal{G}}$ denote the column vector that represents $\mathbf{w}$ under basis~$\mathcal{G}$ \citep[Chapter 4]{meyer2000matrix}. Define the matrix 
\begin{equation*}
    V \coloneqq \sum_{k \in \mathcal{S},i,j \in \mathcal{A}} N_{k,i,j}  [\phi(k,i)-\phi(k,j)]_{\mathcal{G}} [\phi(k,i)-\phi(k,j)]_{\mathcal{G}} ^ \top.
\end{equation*}
Then for any consistent instance $v=(\rho,\mathcal{S},\mathcal{A},\phi,N, \theta)$,  $\eta \in \mathbb{R}$, and $\bold{z} \in {\rm Span} \{\phi(k',i')-\phi(k',j') : (k',i',j')\in \mathcal{S}\times \mathcal{A}^2 \text{ and } N_{k',i',j'}>0  \}$,
\begin{equation}
\label{eq:KL_with_alt_repeat}
\min_{v' \in {\rm Alt}(v,\bold{z},\eta)} \tilde{D}(v,v') = \frac{\eta^2}{\lVert [\bold{z}]_{\mathcal{G}} \rVert^2_{V^{-1}}}.
\end{equation}
and the minimum of \eqref{eq:KL_with_alt_repeat} is attained in $v'=(\rho,\mathcal{S},\mathcal{A},\phi,N, \theta')$ with
$$
\theta' = \theta- \frac{ \eta}{\lVert \bold{[z]_\mathcal{G}} \rVert^2_{V^{-1}}} V^{-1} \bold{z}.
$$
\end{replemma}
\begin{proof}
By definition, we equivalently write down the optimization problem of Eqn.~\eqref{eq:KL_with_alt_repeat} as follows. 
\begin{equation}
\label{eqn:equivalent_KL_with_alt}
\min_{\bold{x} \in \mathbb{R}^d} \sum_{k \in \mathcal{S},i,j \in \mathcal{A}} N_{k,i,j} ( \langle [\phi(k,i)-\phi(k,j)]_{\mathcal{G}}, [\bold{x}]_\mathcal{G} \rangle )^2
\end{equation}
subject to
$$
\langle \bold{x}, \bold{z }  \rangle = \eta.
$$
The Lagrangian of the above constrained optimization problem is,
$$
L(\bold{x}, \lambda) = \sum_{k \in \mathcal{S},i,j \in \mathcal{A}} N_{k,i,j} ( \langle [\phi(k,i)-\phi(k,j)]_{\mathcal{G}}, [\bold{x}]_{\mathcal{G}} \rangle )^2 + \lambda(\langle \bold{x}, \bold{z }  \rangle - \eta).
$$
By solving $\frac{\mathrm{d} L}{\mathrm{d} \lambda} = 0$ and $\frac{\mathrm{d} L}{\mathrm{d} \bold{x}_i} = 0$ for all $i\in [d]$, we attain the minimum of Eqn.~\eqref{eqn:equivalent_KL_with_alt} at
\begin{equation}
\begin{cases}
\lambda = \frac{2\eta}{\lVert \bold{[z]_{\mathcal{G}}} \rVert^2_{V^{-1}}} \nonumber \\
\bold{x} =  -\frac{\eta}{\lVert [\bold{z}]_{\mathcal{G}} \rVert^2_{V^{-1}}} V^{-1} \bold{z} + \bold{g} , \nonumber \\
\end{cases}
\end{equation}
where $\bold{g}$ is any vector that is orthogonal with vector space ${\rm Span} \{\phi(k',i')-\phi(k',j') : (k',i',j')\in \mathcal{S}\times \mathcal{A}^2 \text{ and } N_{k',i',j'}>0  \}$,
which implies
$$
L(\bold{z},\lambda) =\frac{\eta^2}{\lVert [\bold{z}]_{\mathcal{G}} \rVert^2_{V^{-1}}}.
$$
This completes the desired proof.
\end{proof}

\begin{lemma}
\label{lemma:v_prime_small_delta}
Fix any $v = (\rho,\mathcal{S},\mathcal{A},\phi,N, \theta) \in \mathcal{Q}$. Let $$v'=(\rho,\mathcal{S},\mathcal{A},\phi,N, \theta') \in \argmin_{u \in {\rm Alt}(v,\phi(\bar{k}(v),\bar{i}(v))-\phi(\bar{k}(v),i^*_{\bar{k}(v)}) , 2 \Delta_{\bar{k}(v),\bar{i}(v)}(v))}\tilde{D}(v,u).$$ Then,  it holds
\begin{equation}
\label{eq:property}
\begin{cases}
 \langle \phi(k,i^*_k(v))-\phi(k,i), \theta-\theta' \rangle  =   2\Delta_{k,i}(v)\qquad \mbox{for}\  (k,i) = (\bar{k}(v),\bar{i}(v))\\
\lvert \langle \phi(k,i^*_k(v))-\phi(k,i), \theta-\theta' \rangle \rvert \le   \frac{1}{2}\Delta_{k,i}(v)\qquad \forall i\neq i^*_k(v) \mbox{ and } (k,i) \neq (\bar{k}(v),\bar{i}(v))\\
\end{cases}
\end{equation}
Consequently,
\begin{equation}
\begin{cases}
    i^*_k(v) = i^*_k(v') \quad \forall k \neq \bar{k}(v) \\
\bar{k}(v) = \bar{k}(v')  \\
i^*_{\bar{k}(v')}(v') = \bar{i}(v) \\
\end{cases}
\end{equation}
\end{lemma}
\begin{proof}
Fix any $v$ and $v'$ as in Lemma~\ref{lemma:v_prime_small_delta}.
By the definition of $v'$, we can immediately obtain the first inequality of Eqn.~\eqref{eq:property}, i.e.,
\begin{equation}
\label{eq:property1}
\langle \phi(k,i^*_k(v))-\phi(k,i), \theta-\theta' \rangle  =   2\Delta_{k,i}(v)\ \text{for}\  (k,i) = (\bar{k}(v),\bar{i}(v)).
\end{equation}
Then, we will use proof by contradiction to prove the second inequality of Eqn.~\eqref{eq:property}.

Assume that there exists $(\tilde{k},\tilde{i}) \in \mathcal{S} \times \mathcal{A}$ with for all  $\tilde{i}\neq i^*_{\tilde{k}} \text{ and } (\tilde{k},\tilde{i}) \neq (\bar{k}(v),\bar{i}(v))$ such that
\begin{equation}
\label{eq:assuming_property_false}
    \lvert \langle \phi(\tilde{k},i^*_{\tilde{k}})-\phi(\tilde{k},\tilde{i}), \theta-\theta' \rangle \rvert >  \frac{\Delta_{\tilde{k},\tilde{i}}(v)}{2}
\end{equation}
Then, we let $v''= (\rho,\mathcal{S},\mathcal{A},\phi,N, \theta'') $, where 
$$
\theta'' =  \theta+ \frac{(\theta- \theta')2\Delta_{\tilde{k},\tilde{i}}(v) }{\langle \phi(\tilde{k},{i}^*_{\tilde{k}})-\phi(\tilde{k},\tilde{i}), \theta-\theta' \rangle},
$$
which implies that $v'' \in {\rm Alt}(v, \phi(\tilde{k},\tilde{i}) -\phi(\tilde{k},\tilde{i}^*_k), 2\Delta_{\tilde{k},\tilde{i}}(v))$.
Note that from Lemma~\ref{lemma:min_distance_main}, we have
\begin{equation}
 \tilde{D}(v,v')  = \frac{4 \Delta_{\bar{k}(v),\bar{i}(v)}(v)^2 }{\lVert  [ \phi(\bar{k}(v),\bar{i}(v))- \phi(\bar{k}(v),i^*_{\bar{k}(v)}) ]_{\mathcal{G}} \rVert^2_{V^{-1}}},
 \label{eq:d_v_v_prime}
\end{equation}
where $V$ and $\mathcal{G}$ are as defined in Lemma~\ref{lemma:min_distance_main}.
In addition, by definition, we also can get that
\begin{align}
\tilde{D}(v,v'') & =  \left(\frac{2\Delta_{\tilde{k},\tilde{i}}(v)}{\langle \phi(\tilde{k},{i}^*_{\tilde{k}})-\phi(\tilde{k},\tilde{i}), \theta-\theta' \rangle} \right)^2  \cdot \tilde{D}(v,v')  \nonumber \\
& \stackrel{(a)}{<}16  \tilde{D}(v,v')   \nonumber \\
& \stackrel{(b)}{<}   \frac{16 \Delta_{\bar{k}(v),\bar{i}(v)}(v)^2 }{\lVert [ \phi(\bar{k}(v),\bar{i}(v))- \phi(\bar{k}(v),i^*_{\bar{k}(v)}) ]_{\mathcal{G}} \rVert^2_{V^{-1}}} \nonumber \\
& \stackrel{(c)}{<} \frac{64} {\gamma_{\bar{k}(v),\bar{i}(v)}(v) }\label{eq:big_delta_contradition1}
\end{align}
where (a) follows from Eqn.~\eqref{eq:assuming_property_false}  and (b) follows from Eqn.~\eqref{eq:d_v_v_prime}, and (c) follows from Lemma~\ref{lemma:analytic_solutions_of_opt_weights}.

Similarly, from Lemma~\ref{lemma:min_distance_main}, we get that 
\begin{align}
\min_{u \in {\rm Alt}(v,\phi(\tilde{k},\tilde{i})-\phi(\tilde{k},i^*_{\tilde{k}}) , 2 \Delta_{\tilde{k},\tilde{i}}(v))}\tilde{D}(v,u) &= \frac{4 \Delta_{\tilde{k},\tilde{i}}(v)^2 }{\lVert [\phi(\tilde{k},\tilde{i})- \phi(\tilde{k},i^*_{\tilde{k}})]_{\mathcal{G}} \rVert^2_{V^{-1}}} \nonumber \\
& = \frac{4}  {\gamma_{\tilde{k},\tilde{i}} (v)}\label{eq:big_delta_contradition2}
\end{align}
By the fact that $v'' \in  {\rm Alt}(v,\phi(\tilde{k},\tilde{i})-\phi(\tilde{k},i^*_{\tilde{k}}))$ as well as Eqn.~\eqref{eq:contradition1} and Eqn.~\eqref{eq:contradition2}, we further get that
$$
\frac{4}  {\gamma_{\tilde{k},\tilde{i}} (v)} < \frac{64} {\gamma_{\bar{k}(v),\bar{i}(v)} (v)}.
$$
That is, $ 16 {\gamma_{\tilde{k},\tilde{i}} }(v) > \gamma_{\bar{k}(v),\bar{i}(v)}(v)$, which contradicts the fact that $v\in \mathcal{Q}$.

Hence,  for all $(k,i) \in \mathcal{S} \times \mathcal{A}$ with $\forall \tilde{i}\neq i^*_{\tilde{k}} \text{ and } ({k},{i}) \neq (\bar{k}(v),\bar{i}(v))$, we have
\begin{equation*}
    \lvert \langle \phi(k,i^*_k(v))-\phi(\bar{k}(v),\bar{i}(v)), \theta-\theta' \rangle \rvert \ge  2\Delta_{k,i}(v)
\end{equation*}

Consequently, we have
\begin{equation*}
\begin{cases}
    i^*_k(v) = i^*_k(v') \quad \forall k \neq \bar{k}(v) \\
\bar{k}(v) = \bar{k}(v')  \\
i^*_{\bar{k}(v')}(v') = \bar{i}(v) \\
\end{cases}
\end{equation*}
This completes the proof of Lemma~\ref{lemma:v_prime_small_delta}.
\end{proof}

\begin{lemma}
\label{lemma:same_order_hardness}
Fix any $v = (\rho,\mathcal{S},\mathcal{A},\phi,N, \theta) \in \mathcal{Q}$. Let $$v'=(\rho,\mathcal{S},\mathcal{A},\phi,N, \theta') \in \argmin_{u \in {\rm Alt}(v,\phi(\bar{k}(v),\bar{i}(v))-\phi(\bar{k}(v),i^*_{\bar{k}(v)}) , 2 \Delta_{\bar{k}(v),\bar{i}(v)}(v))}\tilde{D}(v,u).$$ Then, it holds 
$$
H(v) \le H(v') \le 8H(v).
$$
\end{lemma}
\begin{proof}

Fix any $v$ and $v'$ as in Lemma~\ref{lemma:v_prime_small_delta}.
Note that from Lemma~\ref{lemma:v_prime_small_delta}, we get that instance $v$ and instance $v'$ share the same optimal action for the state $k \neq \bar{k}$ and the optimal action of state $\bar{k}(v)$ is $\bar{i}(v)$ under the instance $v'$, i.e.,  $i^*_{\bar{k}(v')}(v') = \bar{i}(v) $ and $i^*_k(v) = i^*_k(v')$ for all $k \neq \bar{k}(v)$. Then, by definition, the hardness of instance $v'$ is
\begin{align}
     H(v') & = \max_{k\in \mathcal{S}, i \in \mathcal{A}: i \neq i^*_k(v')}  \   \frac{\gamma_{k,i}(v')}{\Delta_{k,i}^2(v')} \nonumber \\
     &\stackrel{(a)}{=}  \max \left\{ \max_{k\in \mathcal{S}, i \in \mathcal{A}: i \neq i^*_k(v'), k \neq \bar{k}(v)}  \   \frac{\gamma_{k,i}(v)}{\Delta_{k,i}^2(v')}, \max_{i \in \mathcal{A}: i \neq \bar{i}(v), k = \bar{k}(v)}  \frac{\gamma_{k,i}(v')}{\Delta_{k,i}^2(v')} \right\}\nonumber  
\end{align}
where (a) follows the fact that $i^*_k(v) = i^*_k(v') $ for all $ k \neq \bar{k}(v)$ and both instances $v$ and $v'$ share the same parameters of $\phi$ and $N$, which implies  $\gamma_{k,i}(v) = \gamma_{k,i}(v') $ for all $ k \neq \bar{k}(v) $.

Then, by the fact that $\lvert \langle \phi(k,i^*_k(v))-\phi(k,i), \theta-\theta' \rangle \rvert \le   \frac{1}{2}\Delta_{k,i}(v) $ for all $ i\neq i^*_k(v) \text{ and } (k,i) \neq (\bar{k}(v),\bar{i}(v))$ from Lemma~\ref{lemma:v_prime_small_delta}, we get that
\begin{align}
\max_{k\in \mathcal{S}, i \in \mathcal{A}: i \neq i^*_k(v'), k \neq \bar{k}(v)}  \   \frac{\gamma_{k,i}(v)}{\Delta_{k,i}^2(v')}  & \le \max_{k\in \mathcal{S}, i \in \mathcal{A}: i \neq i^*_k(v'), k \neq \bar{k}(v)}  \   \frac{4\gamma_{k,i}(v)}{\Delta_{k,i}^2(v)} \nonumber  \label{eq:prove_hardness_1} \\
& \le 4H(v)
\end{align}
Further, by the fact that  $ \langle \phi(k,i^*_k(v))-\phi(k,i), \theta-\theta' \rangle  =   2\Delta_{k,i}(v)$ for $(k,i) = (\bar{k}(v),\bar{i}(v))$ Lemma~\ref{lemma:v_prime_small_delta}, we get that
\begin{align}
    &\max_{i \in \mathcal{A}: i \neq \bar{i}(v), k = \bar{k}(v)}  \frac{\gamma_{k,i}(v')}{\Delta_{k,i}^2(v')} \nonumber \\
    & = \max \left\{ \frac{\gamma_{\bar{k}(v),\bar{i}(v)}(v)}{\Delta_{\bar{k}(v),\bar{i}(v)}^2(v)}  ,\max_{i \in \mathcal{A}: i \neq \bar{i}(v),i\neq i^*_{\bar{k}(v)}(v) } \frac{\gamma_{\bar{k}(v),i}(v')}{\Delta_{\bar{k}(v),i}^2(v')} \right\} \nonumber \\
    & = \max \left\{ H(v) ,\max_{i \in \mathcal{A}: i \neq \bar{i}(v),i\neq i^*_{\bar{k}(v)}(v) } \frac{\gamma_{\bar{k}(v),i}(v')}{\Delta_{\bar{k}(v),i}^2(v')} \right\} \nonumber \\
    & \stackrel{(a)}{=} \max \left\{ H(v) ,\max_{i \in \mathcal{A}: i \neq \bar{i}(v),i\neq i^*_{\bar{k}(v)}(v) } \frac{ \lVert [\phi(\bar{k}(v),\bar{i}(v)) - \phi(\bar{k}(v),i) ]_{\mathcal{G}}\rVert^2_{V^{-1}}}{\Delta_{\bar{k}(v),i}^2(v')} \right\}  \label{eq:prove_hardness_2} 
\end{align}
where (a) follows from Lemma~\ref{lemma:analytic_solutions_of_opt_weights}, $V$ and $\mathcal{G}$ are as defined in Lemma~\ref{lemma:min_distance_main}.
Similarly, by the fact that $\lvert \langle \phi(k,i^*_k(v))-\phi(k,i), \theta-\theta' \rangle \rvert \le   \frac{1}{2}\Delta_{k,i}(v)\ \  \forall i\neq i^*_k(v) \text{ and } (k,i) \neq (\bar{k}(v),\bar{i}(v))$, we obtain for $i \in \mathcal{A}$ with $ i \neq \bar{i}(v),i\neq i^*_{\bar{k}(v)}(v)$
\begin{align}
  & \frac{ \lVert [\phi(\bar{k}(v),\bar{i}(v)) - \phi(\bar{k}(v),i) ]_{\mathcal{G}} \rVert^2_{V^{-1}}}{\Delta_{\bar{k}(v),i}^2(v')}  \le  \frac{4 \lVert [\phi(\bar{k}(v),\bar{i}(v)) - \phi(\bar{k}(v),i)]_{\mathcal{G}} \rVert^2_{V^{-1}}}{(\Delta_{\bar{k}(v),i}(v) \lor \Delta_{\bar{k}(v),\bar{i}}(v) )^2} \nonumber  \\
   & =  \frac{4 \lVert [\phi(\bar{k}(v),\bar{i}(v)) - \phi(\bar{k}(v),i^*_{\bar{k}(v)}(v))  +\phi(\bar{k}(v),i^*_{\bar{k}(v)}(v))  - \phi(\bar{k}(v),i)]_{\mathcal{G}} \rVert^2_{V^{-1}}}{(\Delta_{\bar{k}(v),i}(v) \lor \Delta_{\bar{k}(v),\bar{i}}(v) )^2} \nonumber \\
   & \le  \frac{4 (\lVert [\phi(\bar{k}(v),\bar{i}(v)) - \phi(\bar{k}(v),i^*_{\bar{k}(v)}(v))   ]_{\mathcal{G}}\rVert^2_{V^{-1}} + \lVert [\phi(\bar{k}(v),i^*_{\bar{k}(v)}(v))  - \phi(\bar{k}(v),i)  ]_{\mathcal{G}} \rVert^2_{V^{-1}}}{(\Delta_{\bar{k}(v),i}(v) \lor \Delta_{\bar{k}(v),\bar{i}}(v) )^2} \nonumber \\
   & \le 8 H(v). \label{eq:prove_hardness_3}
\end{align}
Finally, combining Eqn.~\eqref{eq:prove_hardness_1}, Eqn.~\eqref{eq:prove_hardness_2} and  Eqn.~\eqref{eq:prove_hardness_3}, we complete the proof of Lemma~\ref{lemma:same_order_hardness}.
\end{proof}

With the ingredients of the above lemmas, we are ready to prove Theorem~\ref{thm:lower_bound}.
\begin{proof}[Proof of Theorem~\ref{thm:lower_bound}]
Fix any $v = (\rho,\mathcal{S},\mathcal{A},\phi,N, \theta) \in \mathcal{Q}$ and $$v'=(\rho,\mathcal{S},\mathcal{A},\phi,N, \theta') \in \argmin_{u \in {\rm Alt}(v,\phi(\bar{k}(v),\bar{i}(v))-\phi(\bar{k}(v),i^*_{\bar{k}(v)}) , 2 \Delta_{\bar{k}(v),\bar{i}(v)}(v))}\tilde{D}(v,u).$$
By Lemma~\ref{lemma:same_order_hardness}, we get that $H(v') \le H(v) \le 8H(v')$.
By Lemma~\ref{lemma:min_distance_main}, we get that
\begin{align}
    \tilde{D}(v,v') & = \frac{ \lVert [\phi(\bar{k}(v),i^*_{\bar{k}}(v)) - \phi(\bar{k}(v),\bar{i}(v))]_{\mathcal{G}} \rVert^2_{V^{-1}}}{4\Delta_{\bar{k}(v),\bar{i}(v)}^2(v)}  \nonumber  \\
    & = \frac{1}{4H(v)},  \nonumber 
\end{align}
where $\mathcal{G}$ and $V$ are as defined in Lemma~\ref{lemma:min_distance_main}.
Further, by Lemma~\ref{lemma:kl_divergence_between_instances} and Lemma~\ref{lemma:v_prime_small_delta}, we get that for any $n>0$
\begin{align}
 D_{\rm KL}(P_v^{(n)} \| P_{v'}^{(n)}) &\le {2\exp(2L)} \cdot n \tilde{D}(v,v') \nonumber \\
 & =\frac{n}{2\exp(-2L)H(v)}. \label{eqn:hardness_and_KL}
\end{align}
Then, we let $\bar{\Delta}$ be the minimum suboptimality of state $\bar{k}(v)$ in both instances $v$ and $v'$, i.e.,
$$
\bar{\Delta}= \min_{i \in \mathcal{A}} \left[ \max_{j\in \mathcal{A} } \langle \phi(\bar{k}(v),j) - \phi(\bar{k}(v),i), \theta \rangle  + \max_{j\in \mathcal{A} } \langle \phi(\bar{k}(v),j) - \phi(\bar{k}(v),i), \theta' \rangle \right]
$$
By the definitions of $v$ and $v'$, we can obtain that $\bar{\Delta}>0$.
Then, we get that for any algorithm $\Pi$
\begin{align}
    \mathbb{E}^{\Pi}_v [R_n] +\mathbb{E}^{\Pi}_{v'} [R_n] & \ge  \rho_{\bar{k}(v)} \bar{\Delta} (1- D_{\rm TV} ( P_v^{(n)} , P_{v'}^{(n)} ) ) \\
    & \stackrel{(a)}{\ge}  \frac{1}{2} \rho_{\bar{k}(v)} \bar{\Delta} \exp(-D_{\rm KL}(P_v^{(n)} \| P_{v'}^{(n)})) \nonumber \\
    & \stackrel{(b)}{\ge}  \frac{1}{2} \rho_{\bar{k}(v)} \bar{\Delta} \exp \left(-\frac{n}{2\exp(-2L)H(v)} \right),  \nonumber 
\end{align}
where $D_{\rm TV}(\cdot,\cdot)$ denotes the total variance distance,  and (a) follows from Bretagnolle--Huber inequality~\citep[Lemma 2.6]{tsybakov2009introduction} and  (c) follows from Eqn.~\eqref{eqn:hardness_and_KL}.
Finally, for all sufficiently large $n$, we have 
$$
  \mathbb{E}^{\Pi}_v [R_n] +\mathbb{E}^{\Pi}_{v'} [R_n] \ge \exp \left(-\frac{n}{C_{\rm lo}\cdot H(v)} \right),
$$
which completes the proof of Theorem~\ref{thm:lower_bound}.

\end{proof}

\section{Proof of Upper bounds}
\label{sec:proof_upperbound}
\begin{lemma}
\label{lemma:distance_fE_and_Ef}
Let $Y_n$ be a random variable sampled from the Binomial distribution with  $n$ trials and probability of success  $p\in[1-\beta, \beta]$ for $\beta\in (\frac{1}{2},1)$.  Then,
$$
\lvert \mathbb{E}(f(X_n))- f(\mathbb{E}(X_n)) \rvert \le \frac{3}{(1-\beta)^4 \sqrt{n}}
$$
where $X_n=\frac{Y_n}{n}$ and
\begin{equation}
\label{eq:def_f}
    f(x) = 
\begin{cases}
    \log(\frac{x}{1-x}) \quad \text{ if } x\in (1-\beta, \beta) \\
    \log(\frac{\beta}{1-\beta}) \quad \text{ if } x\ge\beta \\
     \log(\frac{1-\beta}{\beta}) \quad \text{ if } x\le 1-\beta.
\end{cases}
\end{equation}
\end{lemma}
\begin{proof}[Proof of Lemma~\ref{lemma:distance_fE_and_Ef}]
For simplicity of notation, we let $x_0 \coloneqq \mathbb{E}(X_n)$, which implies $x_0 = p \in [1-\beta,\beta]$.
Then, by the smoothness of $f(\cdot)$ on $[1-\beta,\beta]$,  we obtain the equivalent expression of $f(\cdot)$ from the Taylor expansion of $f(\cdot)$ on $[1-\beta,\beta]$, 
\begin{equation}
\label{eq:talor_expansion}
f(x) = 
\begin{cases}
     f(x_0)+f'(x_0)\cdot (x-x_0) + \frac{1}{2}f''(\xi_x) \cdot (x-x_0)^2 \quad \text{ if } x\in [1-\beta, \beta] \\
     f(x_0)+f'(x_0)\cdot (\beta-x_0) + \frac{1}{2}f''(\xi_\beta) \cdot (\beta-x_0)^2 \quad \text{ if } x>\beta \\
     f(x_0)+f'(x_0)\cdot (1-\beta-x_0) + \frac{1}{2}f''(\xi_\beta) \cdot (1-\beta-x_0)^2 \quad \text{ if } x<1-\beta, \\
\end{cases}
\end{equation}
where $\xi_x \in (\min(x,x_0),\max(x,x_0))$ only depends on $x$ and $x_0$ in the Talor expansion.

By using the fact that $f'(x_0)\cdot (\beta-x_0)=f'(x_0)\cdot (\beta-x)+ f'(x_0)\cdot(x-x_0)$ and  $f'(x_0)\cdot (1-\beta-x_0)=f'(x_0)\cdot (1-\beta-x)+ f'(x_0)\cdot(x-x_0)$, we obtain from Eqn.~\eqref{eq:talor_expansion} that 
\begin{align}
   & \lvert \mathbb{E}(f(X_n))-f(\mathbb{E}(X_n)) \rvert  \nonumber \\
   &\le \sup_{x\in (1-\beta,\beta)} \lvert f''(x) \rvert \cdot {\rm Var}(X_n)+  \lvert f'(x_0) \rvert \cdot  \mathbb{E} ([ X_n-\beta ]_+)+ \lvert f'(x_0) \rvert \cdot  \mathbb{E} ([ 1-\beta-X_n ]_+) \label{eqn:bound_fex_efx_1}
\end{align}
where ${\rm Var}(\cdot)$ represents the variance of the random variable $\cdot$ and $[ x ]_+ = \max\{x,0\} $.

In addition, we note that 
\begin{align}
\mathbb{E} ([ X_n-\beta ]_+)  &\le \int_{\beta}^1 (x-\beta) \mathbb{P}(X_n \ge x) \,\mathrm{d} x \nonumber \\ 
& \stackrel{(a)}{\le} \int_{\beta}^1 (x-\beta) \exp(-2n(x-x_0)^2) \,\mathrm{d} x \nonumber  \\
& \stackrel{(b)}{\le} \int_{\beta}^1 (x-\beta) \exp(-n(x-\beta)^2) \,\mathrm{d} x \nonumber \\
& \le \int_{\beta}^1 \frac{1}{\sqrt{n}} \,\mathrm{d} x \nonumber \\
& = \frac{1-\beta}{\sqrt{n}}, \label{eqn:bound_feq_efq_2}
\end{align}
where (a) follows from the fact $X_n$ is a subGaussian random variable with variance proxy $\frac{1}{4n}$, and (b) follows from the fact that $(x-\beta) \exp(-n(x-\beta)^2) \le \frac{1}{\sqrt{n}}$.
Similarly, we also can get that 
\begin{equation}
\mathbb{E} ([ 1-\beta-X_n ]_+)  \le \frac{1-\beta}{\sqrt{n}}.  \label{eqn:bound_feq_efq_3}
\end{equation}

Finally,  
\begin{align}
   & \lvert \mathbb{E}(f(X_n))-f(\mathbb{E}(X_n)) \rvert \nonumber \\
   & \stackrel{(a)}{\le} \sup_{x\in (1-\beta,\beta)} \lvert f''(x) \rvert \cdot {\rm Var}(X_n)+  \lvert f'(x_0) \rvert \cdot  \mathbb{E} ([ X_n-\beta ]_+)+ \lvert f'(x_0) \rvert \cdot  \mathbb{E} ([ 1-\beta-X_n ]_+) \nonumber \\
   & \stackrel{(b)}{\le}  \frac{1}{(1-\beta)^4} \cdot {\rm Var}(X_n)+ f'(x_0) \cdot  \mathbb{E} ([ X_n-\beta ]_+)+f'(x_0) \cdot  \mathbb{E} ([ 1-\beta-X_n ]_+) \nonumber \\
   & \stackrel{(c)}{\le}  \frac{1}{(1-\beta)^4} \cdot \frac{1}{n}+ f'(x_0) \cdot  \mathbb{E} ([ X_n-\beta ]_+)+f'(x_0) \cdot  \mathbb{E} ([ 1-\beta-X_n ]_+) \nonumber \\
   & \stackrel{(d)}{\le}  \frac{1}{(1-\beta)^4} \cdot \frac{1}{n}+ \frac{1}{(1-\beta)^2} \cdot  \mathbb{E} ([ X_n-\beta ]_+)+\frac{1}{(1-\beta)^2} \cdot  \mathbb{E} ([ 1-\beta-X_n ]_+) \nonumber \\
   & \stackrel{(e)}{\le} \frac{1}{(1-\beta)^4} \cdot \frac{1}{n} + \frac{1}{(1-\beta)^2} \cdot \frac{2}{\sqrt{n}} \nonumber \\
   & \le \frac{3}{(1-\beta)^4 \sqrt{n}},
\end{align}
where (a) follows from Eqn.~\eqref{eqn:bound_fex_efx_1},  (b) follows from  ${\rm Var}(X_n) \le \frac{1}{n}$,  (c) follows from $f''(x) = \frac{2x-1}{(x-1)^2x^2} $ for $x\in (\beta,1-\beta)$, (d) follows from $f'(x)= \frac{1}{x-x^2}$ for $x\in (\beta,1-\beta)$, and (e) follows from Eqn.~\eqref{eqn:bound_feq_efq_2} and Eqn.~\eqref{eqn:bound_feq_efq_3}.
\end{proof}

\begin{lemma}
\label{lemma:f_n_sub_gaussian}
Let $Y_n$ be a random variable sampled from the binomial distribution with parameters $n>0$ and $p\in[1-\beta, \beta]$ for $\beta\in (\frac{1}{2},1)$. Let $X_n=\frac{Y_n}{n}$.  Then,
$f(X_n)$ is  subGaussian with 

$$\| f(X_n) \|^2_{\rm vp} \le \frac{C}{n},$$
where $C\le(6\sqrt{2e} \cdot (3\sqrt{\log 2}+1) )^2 \cdot \frac{3}{2 (1-\beta)^4}$ and $f(\cdot)$ is defined in Eqn.~\eqref{eq:def_f}.
\end{lemma}
\begin{proof}
By the definition of $X_n$, we get that 
$$
\| X_n -p \|_{\rm vp}^2 \le \frac{1}{4n}.
$$
Then, by Lemma~\ref{lemma:relation_phi2_and_vp}, we get that
$$
\| X_n -p \|_{\rm \phi_2}^2 \le \frac{3}{2n}.
$$
By the fact that $\lvert f'(x) \rvert \le \frac{1}{(1-\beta)^2}$ for $x \in (1-\beta,\beta)$ and $ f'(x) = 0$ for $x \in (0,1-\beta) \bigcup (\beta,1)$ , we get that
\begin{equation}
    \| f(X_n) -f(p) \|_{\rm \phi_2}^2 \le \frac{3}{2n (1-\beta)^4}.
\end{equation}
Similarly, from Lemma~\ref{lemma:relation_phi2_and_vp}, we further get that
\begin{equation}
    \| f(X_n) -f(p) \|_{\rm vp}^2 \le  (6\sqrt{2e} \cdot (3\sqrt{\log 2}+1) )^2 \cdot \frac{3}{2n (1-\beta)^4}.
\end{equation}
Note that by definition, we have  $\| f(X_n) -f(p) \|_{\rm vp}^2 = \| f(X_n) \|_{\rm vp}^2 $, which completes the proof of Lemma~\ref{lemma:f_n_sub_gaussian}.
\end{proof}
With the ingredients of the above lemmas, we are ready to prove the upper bound of our algorithms
\begin{proof}[Proof of Theorem~\ref{thm:upper_bound}]
Fix any consistent instance $v$ under {\rm RL-LOW}.
By Lemma~\ref{lemma:f_n_sub_gaussian} and Lemma~\ref{lemma:sub-Gaussian-combination}, we get that $\hat{r}_{k,i^*_k,i}$ is subGaussian with variance proxy as 
\begin{equation}
\label{eq:bound_of_vp}
\| \hat{r}_{k,i^*_k,i} \|_{\rm vp}^2 \le \frac{C\gamma_{k,i}}{n},
\end{equation}
where $C\le(6\sqrt{2e} \cdot (3\sqrt{\log 2}+1) )^2 \cdot \frac{3}{2 (1-\beta)^4}$ and $\beta = \frac{\exp(2L)}{1+\exp(2L)}$.

In addition, by Lemma~\ref{lemma:distance_fE_and_Ef}, we get that for any $(k',i',j')\in \mathcal{S} \times \mathcal{A} 
\times \mathcal{A}$ with $N_{k',i',j'}>0$, 
$$
\lvert  \mathbb{E}_v^{\rm RL\mbox{-}LOW} \left [f(B_{k',i',j'}) \right] -   f(\mathbb{E}_v^{\rm RL\mbox{-}LOW} \left[ B_{k',i',j'} \right]) \rvert \le \frac{3}{(1-\beta)^4 \sqrt{n\cdot N_{k',i',j'}}},
$$
which implies for any $(k,i)\in \mathcal{S} \times \mathcal{A}$ with $i \neq i^*_k$, 
\begin{align}
&\left\lvert \sum_{k'\in \mathcal{S}, i'\in \mathcal{A},j'\in \mathcal{A}  } w^{(k,i^*_k,i)}_{k',i',j'} \mathbb{E}_v^{\rm RL\mbox{-}LOW} \left [f(B_{k',i',j'}) \right] -   f(\mathbb{E}_v^{\rm RL\mbox{-}LOW} \left[ \hat{r}_{k,i^*_k,i} \right]) \right\rvert \nonumber \\
&\le \sum_{k'\in \mathcal{S}, i'\in \mathcal{A}, j'\in \mathcal{A} :  N_{k',i',j'}> 0 } \frac{3 \lvert w^{(k,i^*_k,i)}_{k',i',j'} \rvert }{(1-\beta)^4 \sqrt{n\cdot N_{k',i',j'}}}.
\end{align}

That is,
\begin{align}
    \left \lvert \mathbb{E}_v^{\rm RL\mbox{-}LOW} \left[ \hat{r}_{k,i^*_k,i} \right] -r_{k,i^*_k,i}  \right \rvert &\le \sum_{k'\in \mathcal{S}, i'\in \mathcal{A},j'\in \mathcal{A} :  N_{k',i',j'}> 0  } \frac{3 \lvert w^{(k,i^*_k,i)}_{k',i',j'} \rvert }{(1-\beta)^4 \sqrt{n\cdot N_{k',i',j'}}} \\
    & = \frac{3 \tilde{\gamma}_{k,i}}{(1-\beta)^4 \sqrt{n}}
    \label{eq:dis_fe_and_ef},
\end{align}
where we denote $r_{k,i,j} \coloneqq r_{k,i} - r_{k,j}$.

Note that by the definition of $R_n$ we have 
\begin{align}
    R_n & \le  \sum_{(k,i)\in \mathcal{S} \times \mathcal{A}:i\neq i^*_k } \mathds{1}_{\{\hat{r}_{k,i^*_k,i} \le 0\}} \cdot \rho_k \Delta_{k,i}  
 \label{eq:upperbound_basic} \\
    & \le  \sum_{(k,i)\in \mathcal{S} \times \mathcal{A}:i\neq i^*_k } \bigg[  \bigg( \mathds{1}_{\{\hat{r}_{k,i^*_k,i} \le \mathbb{E}_v^{\rm RL\mbox{-}LOW} \left[ \hat{r}_{k,i,i^*_k} \right]- \Delta_{k,i} /2   \}} \nonumber  \\*
    & \hspace{3cm} \lor \mathds{1}_{\{ \lvert \mathbb{E}_v^{\rm RL\mbox{-}LOW} \left[ \hat{r}_{k,i^*_k,i} \right] - \Delta_{k,i} \rvert \ge \Delta_{k,i}/2  \}} \bigg) \cdot \rho_k \Delta_{k,i}  \bigg]   \label{eq:upperbound_basic2} 
\end{align}

 By Lemma~\ref{lemma:tail_bound_of_subgaussian} and  Eqn.~\eqref{eq:bound_of_vp}, we get that for any $(k,i)\in \mathcal{S} \times \mathcal{A}$ with $i \neq i^*_k$, 
\begin{equation}
\label{eq:use_tail_bound}
    \mathbb{P}_v^{\rm RL-LOW} \left( \hat{r}_{k,i^*_k,i} \le \mathbb{E}_v^{\rm RL\mbox{-}LOW} \left[ \hat{r}_{k,i^*_k,i} \right]- \Delta_{k,i} /2 \right) \le \exp\left(-\frac{2n\Delta_{k,i}^2}{C \gamma_{k,i}} \right),
\end{equation}
and by  Eqn.~\eqref{eq:dis_fe_and_ef} we get that $$
\lvert \mathbb{E}_v^{\rm RL\mbox{-}LOW} \left[ \hat{r}_{k,i^*_k,i} \right] -r_{k,i^*_k,i} \rvert \le \Delta_{k,i}/2
$$

for all $n > \frac{18 \tilde{\gamma}_{k,i}^2}{(1-\beta)^8 \Delta_{k,i}}$.
That is, for all sufficiently large $n$, we have 
$$
\mathbb{E}_v^{\rm RL\mbox{-}LOW}\left[R_n \right] \le \sum_{k \in \mathcal{S},i \in \mathcal{A}:i \neq i^*_k}   \rho_k  \Delta_{k,i}\exp\left(-\frac{2n \Delta_{k,i}^2}{C\gamma_{k,i}}\right), 
$$
which further implies that for all sufficiently large $n$, we have
$$
\mathbb{E}_v^{\rm RL\mbox{-}LOW}\left[   R_n \right] \le \exp\left(-\frac{n }{C_{\rm up}\cdot H(v)}\right).
$$
This completes the proof of Theorem~\ref{thm:upper_bound}.
\end{proof}

In addition, we present the proof of Proposition~\ref{props:non_problem_specific_bound} as follows.

\begin{proof}[Proof of Proposition~\ref{props:non_problem_specific_bound}]
Note that by the definition of $R_n$ we have 
\begin{align}
    R_n & \le  \sum_{(k,i)\in \mathcal{S} \times \mathcal{A}:i\neq i^*_k } \mathds{1}_{\{\hat{r}_{k,i^*_k,i} \le 0\}} \cdot \rho_k \Delta_{k,i}  
 \nonumber \\
    & \stackrel{(a)}{\le}  \sum_{(k,i)\in \mathcal{S} \times \mathcal{A}:i\neq i^*_k } \rho_k \bigg[ \mathds{1}_{\{ \Delta_{k,i} < \frac{6 \tilde{\gamma}_{k,i}}{(1-\beta)^4 \sqrt{n}} \}}  \cdot \frac{6 \tilde{\gamma}_{k,i}}{(1-\beta)^4 \sqrt{n}}  \nonumber  \\
    & \hspace{3cm} + \mathds{1}_{\{ \lvert \mathbb{E}_v^{\rm RL\mbox{-}LOW} \left[ \hat{r}_{k,i,i^*_k} \right] -r_{k,i,i^*_k}  \rvert \le \Delta_{k,i}/2  \}} \cdot  \Delta_{k,i}  \bigg]   \nonumber 
\end{align}
where $\beta = \frac{\exp(2L)}{1+\exp(2L)}$, and (a) follows from Eqn.~\eqref{eq:dis_fe_and_ef}.

Hence, by Eqn.~\eqref{eq:use_tail_bound}, we have
\begin{align}
    \mathbb{E}_v^{\rm RL\mbox{-}LOW} [  R_n ]& {\le} \sum_{(k,i)\in \mathcal{S} \times \mathcal{A}:i\neq i^*_k } \rho_k \left[ \frac{6 \tilde{\gamma}_{k,i}}{(1-\beta)^4 \sqrt{n}} +  \exp\left(-\frac{2n\Delta_{k,i}^2}{C \gamma_{k,i}} \right) \Delta_{k,i} \right] \\
    & \le  \sum_{(k,i)\in \mathcal{S} \times \mathcal{A}:i\neq i^*_k } \rho_k \left[ \frac{6 \tilde{\gamma}_{k,i}}{(1-\beta)^4 \sqrt{n}} +  \sqrt{\frac{C \gamma_{k,i}}{2n}} \right] \\
    &= \frac{1}{\sqrt{n}} \sum_{(k,i)\in \mathcal{S} \times \mathcal{A}:i\neq i^*_k } \rho_k \left[ \frac{6 \tilde{\gamma}_{k,i}}{(1-\beta)^4 } +  \sqrt{\frac{C \gamma_{k,i}}{2}} \right]
\end{align}
where $C\le(6\sqrt{2e} \cdot (3\sqrt{\log 2}+1) )^2 \cdot \frac{3}{2 (1-\beta)^4}$.

This completes the desired proof.
\end{proof}

\begin{proof}[Proof of Theorem~\ref{thm:mdp}]
Fix any consistent instance $v$ under {\rm RL-LOW-MDP}.
By Lemma~\ref{lemma:f_n_sub_gaussian} and Lemma~\ref{lemma:sub-Gaussian-combination}, we get that $\hat{r}_{k,i,j}$ is subGaussian with variance proxy as 
\begin{equation}
\label{eq:bound_of_vp2}
\| \hat{r}_{k,i,j} \|_{\rm vp}^2 \le \frac{C\gamma_{k,i,j}}{n},
\end{equation}
where $C\le(6\sqrt{2e} \cdot (3\sqrt{\log 2}+1) )^2 \cdot \frac{3}{2 (1-\beta)^4}$ and $\beta = \frac{\exp(2L)}{1+\exp(2L)}$, and $\gamma_{k,i,j}$ is defined as
$$
\gamma_{k,i,j} \coloneqq \sum_{k' \in \mathcal{S},i',j'\in \mathcal{A} :N_{k',i',j'}> 0} \frac{  (w^{(k,i,j)}_{k',i',j'} )^2 }{N_{k',i',j'}} 
$$

In addition, by Lemma~\ref{lemma:distance_fE_and_Ef}, we get that for any $(k',i',j')\in \mathcal{S} \times \mathcal{A} 
\times \mathcal{A}$ with $N_{k',i',j'}>0$, 
$$
\lvert  \mathbb{E}_v^{\rm RL\mbox{-}LOW\mbox{-}MDP} \left [f(B_{k',i',j'}) \right] -   f(\mathbb{E}_v^{\rm RL\mbox{-}LOW\mbox{-}MDP} \left[ B_{k',i',j'} \right]) \rvert \le \frac{3}{(1-\beta)^4 \sqrt{n\cdot N_{k',i',j'}}},
$$
which implies for any $(k,i,j)\in \mathcal{S} \times \mathcal{A}^2$ with $i \neq j$, 
\begin{align}
&\left\lvert \sum_{k'\in \mathcal{S}, i'\in \mathcal{A},j'\in \mathcal{A}  } w^{(k,i,j)}_{k',i',j'} \mathbb{E}_v^{\rm RL\mbox{-}LOW\mbox{-}MDP} \left [f(B_{k',i',j'}) \right] -   f(\mathbb{E}_v^{\rm RL\mbox{-}LOW\mbox{-}MDP} \left[ \hat{r}_{k,i,j} \right]) \right\rvert \nonumber \\
&\le \sum_{k'\in \mathcal{S}, i'\in \mathcal{A}, j'\in \mathcal{A} :  N_{k',i',j'}> 0 } \frac{3 \lvert w^{(k,i,j)}_{k',i',j'} \rvert }{(1-\beta)^4 \sqrt{n\cdot N_{k',i',j'}}}.
\end{align}

That is,
\begin{align}
    \left \lvert \mathbb{E}_v^{\rm RL\mbox{-}LOW\mbox{-}MDP} \left[ \hat{r}_{k,i,j} \right] -r_{k,i,j}  \right \rvert &\le \sum_{k'\in \mathcal{S}, i'\in \mathcal{A},j'\in \mathcal{A} :  N_{k',i',j'}> 0  } \frac{3 \lvert w^{(k,i,j)}_{k',i',j'} \rvert }{(1-\beta)^4 \sqrt{n\cdot N_{k',i',j'}}} \\
    & = \frac{3 \tilde{\gamma}_{k,i,j}}{(1-\beta)^4 \sqrt{n}}
    \label{eq:dis_fe_and_ef2},
\end{align}
where we denote $$
\tilde{\gamma}_{k,i,j}\coloneqq \sum_{k'\in \mathcal{S}, i'\in \mathcal{A},j'\in \mathcal{A} :  N_{k',i',j'}> 0  } \frac{ \lvert w^{(k,i,j)}_{k',i',j'} \rvert }{ \sqrt{N_{k',i',j'}}}
$$  
and recall that $r_{k,i,j} \coloneqq r_{k,i} - r_{k,j}$.
In addition, we denote $r(\pi) \coloneqq \mathbb{E}_{k \sim d^\pi} [r_{k,\pi(k)}]$ for any MDP policy $\pi$.

Note that by the definition of $\hat{\pi}_{\rm out}$ under {\sc RL-LOW-MDP}, we have 
\begin{align}
    R^{\rm MDP}(\hat{\pi}_{\rm out}) & \le  \sum_{\pi \neq \hat{\pi}_{\rm out}} \mathds{1}_{\{ r(\hat{\pi}_{\rm out}) < r(\pi)\}} \cdot  R^{\rm MDP}(\pi)  
 \nonumber \\
& \le  \sum_{\pi \neq \pi^*}  R^{\rm MDP}(\pi)  \cdot  \mathds{1}_{ \{ \bigcup_{k\in \mathcal{S}} \{ \hat{r}_{k,\pi(k),\pi^*(k))} -  r_{k,\pi(k),\pi^*(k)} \ge   R^{\rm MDP}(\pi)   \} \} } 
 \label{eq:upperbound_basic3} 
\end{align}

 By Lemma~\ref{lemma:tail_bound_of_subgaussian} and  Eqn.~\eqref{eq:bound_of_vp2}, we get that for any $\pi \neq \pi^*$ and $k\in \mathcal{S}$ with $\pi(k) \neq \pi^*(k)$, 
\begin{align}
\label{eq:use_tail_bound_2}
    & \mathbb{P}_v^{\rm RL-LOW-MDP} \bigg( \hat{r}_{k,\pi(k),\pi^*(k)} \ge \mathbb{E}_v^{\rm RL\mbox{-}LOW\mbox{-}MDP} \left[ \hat{r}_{k,\pi(k),\pi^*(k)} \right] +   R^{\rm MDP}(\pi)/2 \bigg)  \nonumber \\
     & \le \exp\left(-\frac{2n ( R^{\rm MDP}(\pi))^2}{C \gamma_{k,\pi(k),\pi^*(k)}} \right),
\end{align}
 and by  Eqn.~\eqref{eq:dis_fe_and_ef2} we get that $$
\big \lvert \mathbb{E}_v^{\rm RL\mbox{-}LOW\mbox{-}MDP} \left[ \hat{r}_{k,i,j} \right] -r_{k,i,j} \big \rvert \le   R^{\rm MDP}(\pi)/2
$$

for all $n > \frac{18 \tilde{\gamma}_{k,i,j}^2}{(1-\beta)^8   R^{\rm MDP}(\pi)/2}$.
That is, for all sufficiently large $n$, we have 
$$
\mathbb{E}_v^{\rm RL\mbox{-}LOW\mbox{-}MDP}\left[R^{\rm MDP}(\hat{\pi}_{\rm out}) \right]  \le  \sum_{\pi \neq \pi^*}  R^{\rm MDP}(\pi)   \sum_{k\in \mathcal{S}: \pi(k) \neq \pi^*(k)}  \exp\left(-\frac{2n ( R^{\rm MDP}(\pi))^2}{C \gamma_{k,\pi(k),\pi^*(k)}} \right)
$$
which further implies that for all sufficiently large $n$, we have
$$
\mathbb{E}_v^{\rm RL\mbox{-}LOW\mbox{-}MDP}\left[R^{\rm MDP}(\hat{\pi}_{\rm out}) \right]  \le \exp\left(-\frac{n }{C_{\rm MDP}\cdot H^{\rm MDP}(v)}\right).
$$
This completes the proof of Theorem~\ref{thm:mdp}.

\end{proof}

\section{Theoretical analysis of {\sc DP-RL-Low}} 
\label{app:DP-RL-Low-app}
\subsection{instance-dependent upper bound}
\begin{lemma}
\label{lemma:distance_fE_and_Ef_dp}
Fix any $\varepsilon>0$ and $\delta>0$. Let $Y_n$ be a random variable sampled from the binomial distribution with number of trials $n\in\mathbb{N}$ and probability of success $p\in[1-\beta, \beta]$ for $\beta\in (\frac{1}{2},1)$.  Then,
$$
\lvert \mathbb{E}(f(\tilde{X}_n))- f(\mathbb{E}(\tilde{X}_n)) \rvert \le \frac{3}{(1-\beta)^4 \sqrt{2n }}  + \frac{4 \sqrt{{{\log(1.25/\delta)}}}}{(1-\beta)^4 {
{(\varepsilon n)}}} ,
$$
where $\tilde{X}_n=\frac{Y_n}{n} + \tilde{\xi}_n$, and $\tilde{\xi}_n$ is an independent Gaussian noise with zero mean and variance of $\frac{2\log(1.25/\delta)}{(\varepsilon n)^2}$, and $f(\cdot)$ is defined in~\eqref{eq:def_f}.

\end{lemma}
\begin{proof}[Proof of Lemma~\ref{lemma:distance_fE_and_Ef_dp}]
For simplicity of notation, we let $x_0 \coloneqq \mathbb{E}(\tilde{X}_n)$, which implies $x_0 = p \in [1-\beta,\beta]$.
Again, by the smoothness of $f(\cdot)$ on $[1-\beta,\beta]$,  we get the Taylor expansion of $f(\cdot)$ on $[1-\beta,\beta]$, 
\begin{equation}
\label{eq:talor_expansion_dp}
f(x) = 
\begin{cases}
     f(x_0)+f'(x_0)\cdot (x-x_0) + \frac{1}{2}f''(\xi_x) \cdot (x-x_0)^2 \quad \text{ if } x\in [1-\beta, \beta] \\
     f(x_0)+f'(x_0)\cdot (\beta-x_0) + \frac{1}{2}f''(\xi_\beta) \cdot (\beta-x_0)^2 \quad \text{ if } x>\beta \\
     f(x_0)+f'(x_0)\cdot (1-\beta-x_0) + \frac{1}{2}f''(\xi_\beta) \cdot (1-\beta-x_0)^2 \quad \text{ if } x<1-\beta, \\
\end{cases}
\end{equation}
where $\xi_x \in (\min(x,x_0),\max(x,x_0))$ that only depends on $x$ and $x_0$ in the Talor expansion.

Similarly, by the fact that $f'(x_0)\cdot (\beta-x_0)=f'(x_0)\cdot (\beta-x)+ f'(x_0)\cdot(x-x_0)$ and  $f'(x_0)\cdot (1-\beta-x_0)=f'(x_0)\cdot (1-\beta-x)+ f'(x_0)\cdot(x-x_0)$, we can get from~\eqref{eq:talor_expansion_dp},
\begin{align}
   & \lvert \mathbb{E}(f(\tilde{X}_n))-f(\mathbb{E}(\tilde{X}_n)) \rvert \nonumber \\
   &\le \sup_{x\in (1-\beta,\beta)} \lvert f''(x) \rvert {\rm Var}(\tilde{X}_n)+  \lvert f'(x_0) \rvert \cdot  \mathbb{E} ([ \tilde{X}_n-\beta ]_+)+ \lvert f'(x_0) \rvert \cdot  \mathbb{E} ([ 1-\beta-\tilde{X}_n ]_+) \label{eqn:bound_eqf_feq_dp1}
\end{align}
where ${\rm Var}(\cdot)$ represents the variance and $[ x ]_+ = \max\{x,0\} $.

In addition, we note that 
\begin{align}
\mathbb{E} ([ \tilde{X}_n-\beta ]_+)  &\le \int_{\beta}^1 (x-\beta) \mathbb{P}(\tilde{X}_n \ge x) \,\mathrm{d}x \nonumber \\ 
& \stackrel{(a)}{\le} \int_{\beta}^1 (x-\beta) \exp(-\frac{(x-x_0)^2}{2(\frac{1}{4n}+\frac{2\log(1.25/\delta)}{(\varepsilon n)^2})}) \,\mathrm{d}x \nonumber  \\
& \le \int_{\beta}^1 (x-\beta) \exp \left(-\left({2n(x-x_0)^2} 
\land \frac{(x-x_0)^2(\varepsilon n)^2}{{4\log(1.25/\delta)}} \right) \right) \,\mathrm{d}x \nonumber  \\
& = \int_{\beta}^1 (x-\beta) \exp \left(-{(x-x_0)^2} \left(2n \land 
\frac{(\varepsilon n)^2}{{4\log(1.25/\delta)}} \right)  \right) \,\mathrm{d}x \nonumber  \\
& \le \int_{\beta}^1 (x-\beta) \exp \left(-{(x-\beta)^2} \left(2n \land
\frac{(\varepsilon n)^2}{{4\log(1.25/\delta)}} \right)  \right) \,\mathrm{d}x \nonumber  \\
& \stackrel{(b)}{\le} \int_{\beta}^1 \sqrt{\frac{1}{{2n \land 
\frac{(\varepsilon n)^2}{{4\log(1.25/\delta)}}}}}\,\mathrm{d}x \nonumber \\
& = \frac{1-\beta}{\sqrt{2n \land 
\frac{(\varepsilon n)^2}{{4\log(1.25/\delta)}}}}, \label{eqn:bound_eqf_feq_dp2}
\end{align}
where (a) follows from the fact that $\tilde{X}_n$ is subGaussian with variance proxy $\frac{1}{4n}+\frac{2\log(1.25/\delta)}{(\varepsilon n)^2}$, and (b) follows from the fact that $x\exp(-x^2 y) < \sqrt{\frac{1}{y}}$ for any $y>0$.
Similarly, we also can get that
\begin{equation}
\label{eqn:bound_eqf_feq_dp3}
\mathbb{E} ([ 1-\beta-X_n ]_+)  \le  \frac{1-\beta}{\sqrt{2n \land 
\frac{(\varepsilon n)^2}{{4\log(1.25/\delta)}}}}.  
\end{equation}

Finally, 
\begin{align}
   & \lvert \mathbb{E}(f(\tilde{X}_n))-f(\mathbb{E}(\tilde{X}_n)) \rvert \nonumber \\
   &\stackrel{(a)}{\le} \sup_{x\in (1-\beta,\beta)} \lvert f''(x) \rvert {\rm Var}(\tilde{X}_n)+  \lvert f'(x_0) \rvert \cdot  \mathbb{E} ([ \tilde{X}_n-\beta ]_+)+ \lvert f'(x_0) \rvert \cdot  \mathbb{E} ([ 1-\beta-\tilde{X}_n ]_+) \nonumber \\
   & \stackrel{(b)}{\le} \frac{1}{(1-\beta)^4} \cdot \left(\frac{1}{n} + \frac{2\log(1.25/\delta)}{(\varepsilon n)^2}\right)+ f'(x_0) \cdot  \mathbb{E} ([ X_n-\beta ]_+)+f'(x_0) \cdot  \mathbb{E} ([ 1-\beta-X_n ]_+) \nonumber \\
   & \stackrel{(c)}{\le} \frac{1}{(1-\beta)^4} \cdot \left(\frac{1}{n} + \frac{2\log(1.25/\delta)}{(\varepsilon n)^2} \right)+ f'(x_0) \cdot  \frac{2}{\sqrt{2n \land 
\frac{(\varepsilon n)^2}{{4\log(1.25/\delta)}}}} \nonumber \\
   & \stackrel{(d)}{\le} \frac{1}{(1-\beta)^4} \cdot \left(\frac{1}{n} + \frac{2\log(1.25/\delta)}{(\varepsilon n)^2} \right) + \frac{1}{(1-\beta)^2} \cdot \frac{2}{\sqrt{2n \land 
\frac{(\varepsilon n)^2}{{4\log(1.25/\delta)}}}} \nonumber \\
& \le \frac{3}{(1-\beta)^4 \sqrt{2n }}  + \frac{4 \sqrt{{{\log(1.25/\delta)}}}}{(1-\beta)^2 {
{(\varepsilon n)}}}   + \frac{2 {{{\log(1.25/\delta)}}}}{(1-\beta)^4 {
{(\varepsilon n)^2}}} \\
& \le \frac{3}{(1-\beta)^4 \sqrt{2n }}  + \frac{4 }{(1-\beta)^4 } \cdot \left( \frac{ \sqrt{{{\log(1.25/\delta)}}}}{ {
{(\varepsilon n)}}}  \lor \frac{ {{{\log(1.25/\delta)}}}}{ {
{(\varepsilon n)^2}}}  \right) \nonumber \\
&  \stackrel{(e)}{\le} \frac{3}{(1-\beta)^4 \sqrt{2n }}  + \frac{4 \sqrt{{{\log(1.25/\delta)}}}}{(1-\beta)^4 {
{(\varepsilon n)}}}  \nonumber 
\end{align}
where (a) follows from Eqn.~\eqref{eqn:bound_eqf_feq_dp1}, (b) follows from the fact that ${\rm Var}(\tilde{X}_n) \le \frac{1}{n} + \frac{2\log(1.25/\delta)}{(\varepsilon n)^2}$ and that $f''(x) = \frac{2x-1}{(x-1)^2x^2} $, (c) follows from Eqn.~\eqref{eqn:bound_eqf_feq_dp2} and Eqn.~\eqref{eqn:bound_eqf_feq_dp3},  (d) follows from the fact that $f'(x)= \frac{1}{x-x^2}$, and (e) follows the fact that $\lvert \mathbb{E}(f(\tilde{X}_n))-f(\mathbb{E}(\tilde{X}_n)) \rvert \le \frac{4 }{(1-\beta)^4 }$ and that $\sqrt{x} \ge x$ for any $x \in (0,1]$
\end{proof}

\begin{lemma}
\label{lemma:f_n_sub_gaussian_DP}
Fix any $\varepsilon>0$ and $\delta>0$. Let $Y_n$ be a random variable sampled from the binomial distribution with number of trials $\in\mathbb{N}$ and probability of success $p\in[1-\beta, \beta]$ for $\beta\in (\frac{1}{2},1)$. Let $\tilde{X}_n=\frac{Y_n}{n} + \tilde{\xi}_n$, and $\tilde{\xi}_n$ is an independent Gaussian noise with zero mean and variance of $\frac{2\log(1.25/\delta)}{(\varepsilon n)^2}$.  Then,
$f(\tilde{X}_n)$ is  subGaussian with 

$$\| f(\tilde{X}_n) \|^2_{\rm vp} \le C \cdot
\left( \frac{1}{n} + \frac{8\log(1.25/\delta)}{(\varepsilon n)^2} \right),$$
where $C\le(6\sqrt{2e} \cdot (3\sqrt{\log 2}+1) )^2 \cdot \frac{3}{2 (1-\beta)^4}$ and $f(\cdot)$ is defined in~\eqref{eq:def_f}.
\end{lemma}
\begin{proof}
By the definition of $X_n$, we get that 
$$
\| \tilde{X}_n -p \|_{\rm vp}^2 \le \frac{1}{4n} + \frac{2\log(1.25/\delta)}{(\varepsilon n)^2} .
$$
Then, by Lemma~\ref{lemma:relation_phi2_and_vp}, we get that
$$
\| \tilde{X}_n -p \|_{\rm \phi_2}^2 \le \frac{3}{2n} + \frac{12\log(1.25/\delta)}{(\varepsilon n)^2} .
$$
By the fact that $\lvert f'(x) \rvert \le \frac{1}{(1-\beta)^2}$, we get that
\begin{equation}
    \| f(\tilde{X}_n) -f(p) \|_{\rm \phi_2}^2 \le \left(\frac{3}{2n} + \frac{12\log(1.25/\delta)}{(\varepsilon n)^2}\right)\cdot \frac{1}{(1-\beta)^4}.
\end{equation}
Similarly, from Lemma~\ref{lemma:relation_phi2_and_vp}, we further get that
\begin{equation}
    \| f(\tilde{X}_n) -f(p) \|_{\rm vp}^2 \le  (6\sqrt{2e} \cdot (3\sqrt{\log 2}+1) )^2 \cdot \frac{3}{2 (1-\beta)^4} \left(\frac{1}{n} + \frac{8\log(1.25/\delta)}{(\varepsilon n)^2}\right).
\end{equation}
Note that by definition, we have  $\| f(\tilde{X}_n) -f(p) \|_{\rm vp}^2 = \| f(\tilde{X}_n) \|_{\rm vp}^2 $, which completes the proof of Lemma~\ref{lemma:f_n_sub_gaussian}.
\end{proof}

\begin{proof}[Proof of Theorem~\ref{thm:dp_upperbound}]

Fix any consistent instance $v$ and $n>0$ under {\rm DP-RL-LOW}.
By lemma~\ref{lemma:f_n_sub_gaussian} and lemma~\ref{lemma:sub-Gaussian-combination}, we get that $\tilde{r}_{k,i^*_k,i}$ is subGaussian with variance proxy to be,
\begin{equation}
\label{eq:bound_of_vp_dp}
\| \tilde{r}_{k,i^*_k,i} \|_{\rm vp}^2 \le C \cdot \left(\frac{\gamma_{k,i} }{n} + \frac{8\gamma_{k,i}^{\rm DP}\log(1.25/\delta)}{(\varepsilon n)^2}  \right),
\end{equation}
where $C\le(6\sqrt{2e} \cdot (3\sqrt{\log 2}+1) )^2 \cdot \frac{3}{2 (1-\beta)^4}$ and $\beta = \frac{\exp(2L)}{1+\exp(2L)}$.

In addition, by Lemma~\ref{lemma:distance_fE_and_Ef_dp}, we get that for any $(k',i',j')\in \mathcal{S} \times \mathcal{A}^2$ with $N_{k',i',j'}>0$, 
\begin{align}
& \left \lvert  \mathbb{E}_v^{\rm DP\mbox{-}RL\mbox{-}LOW} \left [f(\tilde{B}_{k',i',j'}) \right] -   f \left(\mathbb{E}_v^{\rm DP\mbox{-}RL\mbox{-}LOW} \left[ \tilde{B}_{k',i',j'} \right]\right) \right\rvert  \nonumber \\
&\le \frac{3}{(1-\beta)^4 \sqrt{2n N_{k',i',j'} }}  + \frac{4 \sqrt{{{\log(1.25/\delta)}}}}{(1-\beta)^4 {
{(\varepsilon n N_{k',i',j'})}}}  ,
\end{align}
which implies for any $(k,i) \in \mathcal{S} \times \mathcal{A} $ with $i \neq i^*_k$, 
\begin{align}
&\left\lvert \sum_{k'\in \mathcal{S}, i'\in \mathcal{A},j'\in \mathcal{A}  } w^{(k,i^*_k,i)}_{k',i',j'} \mathbb{E}_v^{\rm DP\mbox{-}RL\mbox{-}LOW} \left [f(B_{k',i',j'}) \right] -   f(\mathbb{E}_v^{\rm DP\mbox{-}RL\mbox{-}LOW} \left[ B_{k',i',j'} \right]) \right\rvert \nonumber \\
&\le \sum_{k'\in \mathcal{S}, i'\in \mathcal{A},j'\in \mathcal{A} : N_{k',i',j'} \neq 0 } 
\lvert w^{(k,i^*_k,i)}_{k',i',j'} \rvert \left( \frac{3}{(1-\beta)^4 \sqrt{2nN_{k',i',j'} }}  + \frac{4 \sqrt{{{\log(1.25/\delta)}}}}{(1-\beta)^4 {
{(\varepsilon nN_{k',i',j'})}}} \right). \label{eqn:weight_sum_ef_fe}
\end{align}
Recall that we denote $r_{k,i^*k,i} = r_{k,i^*_k}-r_{k,i}$. Then, from Eqn.~\eqref{eqn:weight_sum_ef_fe}, we get
\begin{align}
    &\left \lvert \mathbb{E}_v^{\rm DP\mbox{-}RL\mbox{-}LOW} \left[ \tilde{r}_{k,i^*_k,i} \right] -r_{k,i^*_k,i}  \right \rvert  \nonumber \\
    &\le \sum_{k'\in \mathcal{S}, i'\in \mathcal{A},j'\in \mathcal{A} : N_{k',i',j'} \neq 0 } \lvert  w^{(k,i,j)}_{k',i',j'} \rvert \left( \frac{3}{(1-\beta)^4 \sqrt{2nN_{k',i',j'} }}  + \frac{4 \sqrt{{{\log(1.25/\delta)}}}}{(1-\beta)^4 {
{(\varepsilon nN_{k',i',j'})}}} \right)\\
    & 
    \le \frac{3 \tilde{\gamma}_{k,i}}{(1-\beta)^4 \sqrt{n}} + \frac{4 \tilde{\gamma}^{DP}_{k,i} \sqrt{{{\log(1.25/\delta)}}}}{(1-\beta)^4 {
{(\varepsilon n)}}},
    \label{eq:dis_fe_and_ef_dp}
\end{align}
where $\tilde{\gamma}^{\mathrm{DP}}_{k,i} \coloneqq  \sum_{k' \in \mathcal{S}, i'\in\mathcal{A},j' \in \mathcal{A} :N_{k',i',j'}\neq 0 }   \frac{\lvert w^{(k,i^*_k,i)}_{k',i',j'} \rvert }{{N_{k',i',j'}}} $

Note that by the definition of $R_n$, under {\sc DP-RL-LOW} we have 
\begin{align}
    R_n & \le  \sum_{(k,i)\in \mathcal{S} \times \mathcal{A}:i\neq i^*_k } \mathds{1}_{\{\tilde{r}_{k,i^*_k,i} <0\}} \cdot \rho_k \Delta_{k,i}  
 \label{eq:upperbound_basic_dp} \\
    & \le  \sum_{(k,i)\in \mathcal{S} \times \mathcal{A}:i\neq i^*_k } \bigg[ \big( \mathds{1}_{\{\tilde{r}_{k,i^*_k,i} \le \mathbb{E}_v^{\rm DP\mbox{-}RL\mbox{-}LOW} \left[ \tilde{r}_{k,i^*_k,i} \right]- \Delta_{k,i} /2   \}} \nonumber  \\*
    & \hspace{3cm} \lor \mathds{1}_{\{ \lvert \mathbb{E}_v^{\rm DP\mbox{-}RL\mbox{-}LOW} \left[ \tilde{r}_{k,i^*_k,i} \right] -r_{k,i^*_k,i} \rvert \le \Delta_{k,i}/2  \}} \big) \cdot \rho_k \Delta_{k,i}  \bigg]   \label{eq:upperbound_basic2_dp} 
\end{align}

 By Lemma~\ref{lemma:tail_bound_of_subgaussian} and ~\eqref{eq:bound_of_vp_dp}, we get that for any $(k,i)\in \mathcal{S} \times \mathcal{A}$ with $i \neq i^*_k$, 
\begin{align}
\label{eq:use_tail_bound2}
    &\mathbb{P}_v^{\rm DP-RL-LOW} \left( \tilde{r}_{k,i^*_ki,} \le \mathbb{E}_v^{\rm DP\mbox{-}RL\mbox{-}LOW} \left[ \tilde{r}_{k,i^*_k,i} \right]- \Delta_{k,i} /2 \right)  \nonumber \\
    &\le \exp\left(-\frac{2\Delta_{k,i}^2}{C \cdot (\gamma_{k,i}/n+{8\gamma_{k,i}^{\rm DP}\log(1.25/\delta)}/{(\varepsilon n)^2})} \right),
\end{align}
and by~\eqref{eq:dis_fe_and_ef_dp} we get that $$
\lvert \mathbb{E}_v^{\rm DP\mbox{-}RL\mbox{-}LOW} \left[ \tilde{r}_{k,i^*_k,i} \right] -r_{k,i^*_k,i}  \rvert \le \Delta_{k,i}/2
$$
for all $n > \frac{12 \tilde{\gamma}_{k,i}^2}{(1-\beta)^8 \Delta_{k,i}^2} + \frac{8 \tilde{\gamma}^{DP}_{k,i} \sqrt{{{\log(1.25/\delta)}}}}{(1-\beta)^4 {
{(\varepsilon \Delta_{k,i})}}} $.
That is, for all sufficiently large $n$, we have 
$$
\mathbb{E}_v^{\rm RL\mbox{-}LOW}\left( R_n \right) \le \sum_{k \in \mathcal{S},i \in \mathcal{A}: i \neq i^*_k}   \rho_k  \Delta_{k,i}\exp\left(-\frac{2\Delta_{k,i}^2}{C (\gamma_{k,i}/n+{8\gamma_{k,i}^{\rm DP}\log(1.25/\delta)}/{(\varepsilon n)^2})} \right), 
$$
which further implies that for all sufficiently large $n$, there exists a global constant $C_{\rm DP}$,  we have
$$
\mathbb{E}_v^{\rm RL\mbox{-}LOW}\left( R_n \right) \le \exp\left(-C_{\rm DP} \cdot \left(\frac{n }{H(v)} \land \left(\frac{n }{H_{\rm DP}^{(\varepsilon,\delta))}(v) }\right)^2 \right) \right).
$$
This completes the proof of Theorem~\ref{thm:dp_upperbound}
\end{proof}
{\color{black}
\subsection{Worst-case Upper Bound}
In this section, we derive the worst-case upper bound of {\sc DP-RL-LOW}, which yields the form of  $O(\frac{1}{\sqrt{n}}+\frac{\sqrt{\log(1.25/\delta)}}{\varepsilon n})$ for the dependency of $n$,$\varepsilon$ and $\delta$. This form resembles that in~\citet{qiao2024offline} without pairwise comparisons.
\label{sub:non_problem_specific_bound_dp}
\begin{proposition}{(Worst-Case Upper Bound for DP-RL-LOW)}
\label{props:non_problem_specific_bound_dp}
For any consistent instance $v$ and for all $n\ge 1$, 
 \begin{equation}
 \!\mathbb{E}^{\rm DP-RL\mbox{-}LOW}_v \left[R_n \right] \!\le\! C_{\rm WDP} \cdot \left( \frac{\sum\limits_{ \substack{ k ,i  : i \neq i^*_k} }  \rho_k (\sqrt{\gamma_{k,i}}\!+\! \tilde{\gamma}_{k,i}) }{\sqrt{n}} + \frac{\sum\limits_{ \substack{ k ,i  : i \neq i^*_k} }  \rho_k (\sqrt{\gamma_{k,i}^{\rm DP} }\!+\! \tilde{\gamma}^{\rm DP}_{k,i} )\sqrt{\log(1.25/\delta)} }{\epsilon{n}} \right)
 \end{equation}
where $C_{\rm WDP}>0$ is a universal constant.
\end{proposition}

\begin{proof}[Proof of Proposition~\ref{props:non_problem_specific_bound_dp}]
Note that by the definition of $R_n$ we have 
\begin{align}
    R_n & \le  \sum_{(k,i)\in \mathcal{S} \times \mathcal{A}:i\neq i^*_k } \mathds{1}_{\{\tilde{r}_{k,i^*_k,i} \le 0\}} \cdot \rho_k \Delta_{k,i}  
 \nonumber \\
    & \stackrel{(a)}{\le}  \sum_{(k,i)\in \mathcal{S} \times \mathcal{A}:i\neq i^*_k } \rho_k \bigg[ \mathds{1}_{\{ \Delta_{k,i} < \frac{6 \tilde{\gamma}_{k,i}}{(1-\beta)^4 \sqrt{n}} + \frac{48 \tilde{\gamma}^{\rm DP}_{k,i} \sqrt{{{\log(1.25/\delta)}}}}{(1-\beta)^4 {
{(\varepsilon n)}}} \}}  \cdot \big(\frac{6 \tilde{\gamma}_{k,i}}{(1-\beta)^4 \sqrt{n}} + \frac{8 \tilde{\gamma}^{DP}_{k,i} \sqrt{{{\log(1.25/\delta)}}}}{(1-\beta)^4 {
{(\varepsilon n)}}} \big)  \nonumber  \\
    & \hspace{3cm} + \mathds{1}_{\{ \lvert \mathbb{E}_v^{\rm DP-RL\mbox{-}LOW} \left[ \tilde{r}_{k,i,i^*_k} \right] -r_{k,i,i^*_k}  \rvert \le \Delta_{k,i}/2  \}} \cdot  \Delta_{k,i}  \bigg]   \label{eq:dp_split_1} 
\end{align}
where $\beta = \frac{\exp(2L)}{1+\exp(2L)}$, and (a) follows from Eqn.~\eqref{eq:dis_fe_and_ef_dp}.

Hence, by Eqn.~\eqref{eq:dp_split_1} and~\eqref{eq:use_tail_bound2}, we have
\begin{align}
    \mathbb{E}_v^{\rm DP-RL\mbox{-}LOW} [  R_n ]& {\le} \sum_{(k,i)\in \mathcal{S} \times \mathcal{A}:i\neq i^*_k } \rho_k \bigg[ \big(\frac{6 \tilde{\gamma}_{k,i}}{(1-\beta)^4 \sqrt{n}} + \frac{8 \tilde{\gamma}^{DP}_{k,i} \sqrt{{{\log(1.25/\delta)}}}}{(1-\beta)^4 {
{(\varepsilon n)}}} \big)   \nonumber \\
    & \hspace{.3in} + \exp\left(-\frac{2\Delta_{k,i}^2}{C \cdot (\gamma_{k,i}/n+{8\gamma_{k,i}^{\rm DP}\log(1.25/\delta)}/{(\varepsilon n)^2})} \right) \Delta_{k,i} \bigg] \nonumber \\
    & {\le} \sum_{(k,i)\in \mathcal{S} \times \mathcal{A}:i\neq i^*_k } \rho_k \bigg[ \big(\frac{6 \tilde{\gamma}_{k,i}}{(1-\beta)^4 \sqrt{n}} + \frac{8 \tilde{\gamma}^{DP}_{k,i} \sqrt{{{\log(1.25/\delta)}}}}{(1-\beta)^4 {
{(\varepsilon n)}}} \big)   \nonumber \\
    & \hspace{.3in} + \sqrt{\frac{C \cdot \gamma_{k,i}}{2n}}+\sqrt{\frac{C \cdot {8\gamma_{k,i}^{\rm DP}\log(1.25/\delta)}}{2(\varepsilon n)^2}}  \bigg] \nonumber \\
& {=} \sum_{(k,i)\in \mathcal{S} \times \mathcal{A}:i\neq i^*_k } \rho_k \bigg[ \frac{1}{\sqrt{n}} \cdot \left(\frac{6 \tilde{\gamma}_{k,i}}{(1-\beta)^4 }+\sqrt{\frac{C \cdot \gamma_{k,i}}{2}} \right)   \nonumber \\
    & \hspace{.3in} + \frac{1}{\varepsilon n} \cdot \left(\frac{8 \tilde{\gamma}^{\rm DP}_{k,i} \sqrt{{{\log(1.25/\delta)}}}}{(1-\beta)^4 {
}} +\sqrt{\frac{C \cdot {8\gamma_{k,i}^{\rm DP}\log(1.25/\delta)}}{2}} \right) \bigg] \nonumber \\
\end{align}
where $C\le(6\sqrt{2e} \cdot (3\sqrt{\log 2}+1) )^2 \cdot \frac{3}{2 (1-\beta)^4}$.

This completes the desired proof.
\end{proof}

}

\end{document}